\documentclass[conference]{IEEEtran}
\usepackage{graphicx}
\usepackage{amsthm}
\usepackage{epsfig}
\usepackage{latexsym}
\usepackage{amsfonts}
\usepackage{here}
\usepackage{rawfonts}
\usepackage[latin1]{inputenc}
\usepackage[T1]{fontenc}
\usepackage{calc}
\usepackage{capitalgreekitalic}
\usepackage{url}
\usepackage{enumerate}
\usepackage{color}
\usepackage[tbtags]{amsmath}
\usepackage{amssymb}
\usepackage{upref}
\usepackage{epic,eepic}
\usepackage{times}
\usepackage{dsfont}
\usepackage{comment}
\usepackage{cite}
\usepackage{bbm} 
\newtheorem{corollary}{Corollary}
\newtheorem{theorem}{\bf Theorem}

\usepackage{dsfont}
\newcounter{step}
\newlength{\totlinewidth}
\newenvironment{algorithm}{%
  \rule{\linewidth}{1pt}
  \begin{list}{}%
    {\usecounter{step}%
      \settowidth{\labelwidth}{\textbf{Step 2:}}%
      \setlength{\leftmargin}{\labelwidth}%
      \setlength{\topsep}{-2pt}%
      \addtolength{\leftmargin}{\labelsep}%
      \addtolength{\leftmargin}{2mm}%
      \setlength{\rightmargin}{2mm}%
      \setlength{\totlinewidth}{\linewidth}%
      \addtolength{\totlinewidth}{\leftmargin}%
      \addtolength{\totlinewidth}{\rightmargin}%
      \setlength{\parsep}{0mm}%
      \raggedright}}%
  {\end{list}%
  \rule{\linewidth}{1pt}}
\newcounter{substep}

  {\end{list}}

\newlength{\aligntop}
\newlength{\alignbot}
 \makeatletter
\renewenvironment{align}{%
  \vspace{\aligntop}
  \start@align\@ne\st@rredfalse\m@ne
}{%
  \math@cr \black@\totwidth@
  \egroup
  \ifingather@
    \restorealignstate@
    \egroup
    \nonumber
    \ifnum0=`{\fi\iffalse}\fi
  \else
    $$%
  \fi
  \ignorespacesafterend%
  \vspace{\alignbot}\par\noindent
} \makeatother

\IEEEoverridecommandlockouts

\usepackage{algorithm}%,algpseudocode}
\usepackage{algorithmic}
\makeatletter
\newcommand\semihuge{\@setfontsize\semihuge{19.3}{25}}
\makeatother

\makeatletter
\newcommand\semismall{\@setfontsize\semihuge{12.4}{15}}
\makeatother

\usepackage{subfigure}
\begin{document}

\title{\Huge Convergence Time Optimization for Federated Learning over Wireless Networks}

\author{{Mingzhe Chen}, \emph{Member, IEEE}, H. Vincent Poor, \emph{Life Fellow, IEEE}, Walid Saad, \emph{Fellow, IEEE},\\ and Shuguang Cui, \emph{Fellow, IEEE} \vspace*{-2em}\\ 
%\IEEEauthorrefmark{5}Department of Electrical and Computer Engineering University of California, Davis, CA, USA, Email: \protect\url{sgcui@ucdavis.edu}.\\
%\IEEEauthorrefmark{3}\small Mathematical and Algorithmic Sciences Lab, Huawei France R \& D, Paris, France, \\Email: merouane.debbah@huawei.com.\\
%\IEEEauthorrefmark{4}\small Department of Computer Science and Engineering, Kyung Hee University, Yongin, South Korea, \\Email: \protect\url{cshong@khu.ac.kr.}\\
%}\vspace*{-4em}
\thanks{{
This work was supported in part by the National Key R\&D Program of China with grant No. 2018YFB1800800, by the Key Area R\&D Program of Guangdong Province with grant No. 2018B030338001, by Shenzhen Outstanding Talents Training Fund, and by Guangdong Research Project No. 2017ZT07X152, and in part by the U.S. National Science Foundation under Grants CCF-1908308, CCF-0939370, and CCF-1513915.}}
\thanks{M. Chen and H. V. Poor are with the Department of Electrical and Computer Engineering, Princeton University, Princeton, NJ, 08544, USA, Emails: \protect{mingzhec@princeton.edu}, \protect\url{poor@princeton.edu}.}
\thanks{W. Saad is with the Wireless@VT, Bradley Department of Electrical and Computer Engineering, Virginia Tech, Blacksburg, VA, 24060, USA, Email: \protect{walids@vt.edu}.}
\thanks{S. Cui is with the Shenzhen Research Institute of Big Data and School of Science and Engineering, the Chinese University of Hong Kong, Shenzhen, 518172, China, Email: \protect{shuguangcui@cuhk.edu.cn} }
 }

\maketitle
%
%%
%\vspace{0cm}
\begin{abstract}
In this paper, the convergence time of federated learning (FL), when deployed over a realistic wireless network, is studied.
In particular, a wireless network is considered in which wireless users transmit their local FL models (trained using their locally collected data) to a base station (BS). The BS, acting as a central controller, generates a global FL model using the received local FL models and broadcasts it back to all users. Due to the limited number of resource blocks (RBs) in a wireless network, only a subset of users can be selected to transmit their local FL model parameters to the BS at each learning step. Moreover, since each user has unique training data samples, the BS prefers to include all local user FL models to generate a converged global FL model. Hence, the FL {\color{black}training loss} and convergence time will be significantly affected by the user selection scheme.   
%Therefore, it is necessary to design an appropriate user selection scheme that enables users of higher importance to be selected more frequently. 
{\color{black}Therefore, it is necessary to design an appropriate user selection scheme that can select the users who can contribute toward improving the FL convergence speed more frequently.} 
This joint learning, wireless resource allocation, and user selection problem is formulated as an optimization problem whose goal is to minimize the FL convergence time and the FL {\color{black}training loss}. 
 To solve this problem, a probabilistic user selection scheme is proposed such that the BS is connected to the users whose local FL models have significant effects on the global FL model with high probabilities. Given the user selection policy, the uplink RB allocation can be determined. To further reduce the FL convergence time, artificial neural networks (ANNs) are used to estimate the local FL models of the users that are not allocated any RBs for local FL model transmission at each given learning step, which enables the BS to improve the global model, the FL convergence speed, and the {\color{black}training loss.}
Simulation results show that the proposed approach
 can reduce the FL convergence time by up to $56\%$ and improve the accuracy of identifying handwritten digits by up to $3\%$, compared to a standard FL algorithm.
\end{abstract}
{\small \emph{Index Terms} --- Federated learning; wireless resource allocation; probabilistic user selection; artificial neural networks.}
{\renewcommand{\thefootnote}{\fnsymbol{footnote}}
\footnotetext{A preliminary version of this work \cite{chen2020ICC} appears in the Proceedings of the 2020 IEEE International Conference on Communications.}}

\section{Introduction}
To train traditional machine learning algorithms for data analysis and inference, a central controller requires all users' training data samples. However, in wireless networks, it is impractical for wireless users to transmit their training data samples to such a central controller due to privacy concerns and the nature of limited wireless resources \cite{saad2019vision}. In order to train a machine learning model without collection of all of the users' training data samples, \emph{federated learning} (FL) has been proposed in \cite{bonawitz2019towards}. 
 FL is, in essence, a distributed machine learning algorithm that enables users to collaboratively learn a shared machine learning model while keeping their collected data on their devices. {\color{black} For wireless communications, FL enables many applications such as Internet of Things (IoT) (e.g., ultra reliable
low latency services), extended reality (XR), and autonomous vehicles. For example,
in autonomous driving, FL enables vehicles to complete the environmental perception and
object detection tasks in milliseconds. %Meanwhile, FL enables the massive multiple input
%multiple output (MIMO) systems to train and update deep neural network parameters for
%channel estimation and beamforming in an online manner. 
However,  implementing FL and its applications over real-world wireless networks requires edge devices and the BS to repeatedly exchange their FL model parameters. Due to limited wireless resources such as bandwidth, in a wireless network, only a subset of devices can engage in FL. Meanwhile, FL model parameters that are transmitted from the users to a BS are subject to errors and delays caused by the wireless channel which affects the learning performance. For example, when an FL algorithm is implemented over a wireless network, its convergence time not only depends on the number of training steps, but also depends on the ML model parameter transmission time at each training step. Meanwhile, as the number of neurons in the ML model increases, the computation and communication delay of each device will significantly increase thus increasing the FL convergence time. In consequence, energy-constrained devices that are running time-sensitive applications (e.g. virtual reality and ultra-reliable low-latency communication) may not be able or want to wait to complete such a long time training process. Therefore, it is necessary to consider the optimization of wireless networks to reduce FL training time and training loss.}
 
% implementing FL over wireless networks faces several challenges \cite{8770530,niknam2019federated,park2018wireless}, such as selection of users for FL, energy efficiency of FL implementation, computational resource allocation for training FL models at edge devices, spectrum resource allocation for FL parameter transmission, and design of communication-efficient FL. {\color{black}In particular, due to energy constraints, each wireless user can only perform FL within a short period of time. Meanwhile, as the number of neurons of an FL model and the number of users that engage in FL increase, the convergence time will significantly increase. Hence, due to high convergence time, the standard framework \cite{bonawitz2019towards} that only minimizes the training loss may not be suitable for training neural networks that consist of thousands of neurons, time-sensitive applications (e.g., virtual reality), and a large scale network such as the Internet of Things (IoT). In consequence, for real-world FL applications over wireless networks, we must design an FL algorithm with a short convergence time.} 
 
% decreasing the data size of the FL model parameters that are transmitted over wireless networks.
 
Recently, the works in \cite{zhu2018towards,8733825,8664630,chen2019joint,8737464,yang2019energy,8851408,vu2019cell,ren2019accelerating,8851249,feng2018joint,skatchkovsky2019federated} has studied important problems related to the deployment of FL over wireless networks. In \cite{zhu2018towards}, the authors provided a new set of design principles for wireless communication and edge learning; however, they did not provide any mathematically rigorous result for minimizing {\color{black}FL training loss} over wireless networks. The authors  in \cite{8733825} proposed a blockchain-based FL scheme that enables edge devices to train FL without sending FL model parameters to a central controller. The works in \cite{8664630} and \cite{chen2019joint} optimized the FL performance with communication constraints such as limited spectrum and computational resources. In \cite{8737464} and \cite{yang2019energy}, the authors investigated the tradeoff between the FL convergence time and users' energy consumption. 
The authors in \cite{8851408} developed an echo state network based FL to predict locations and orientations of wireless virtual reality users so as to minimize their breaks in presence. The work in \cite{vu2019cell} developed a cell-free massive multiple-input multiple-output (MIMO) system for the implementation of FL over wireless networks. In \cite{ren2019accelerating}, the authors define a metric to evaluate the FL performance and optimized the batchsize selection and communication resource allocation for the acceleration of the FL training process. The work in \cite{8851249} derived an analytical model to characterize the {\color{black}training loss} of FL in wireless networks and evaluate the effectiveness of different scheduling policies.  
The authors in \cite{feng2018joint} studied the use of the relay network to construct a cooperative communication platform for supporting FL model transmission. In \cite{skatchkovsky2019federated}, the authors developed a federated learning based spiking neural network. 
However, most of these existing works \cite{zhu2018towards,8733825,8664630,chen2019joint,8737464,8851408,vu2019cell,ren2019accelerating,8851249,feng2018joint,yang2019energy,skatchkovsky2019federated} whose goal is to minimize the FL convergence time must sacrifice the {\color{black}training loss} of the FL algorithm. For example, in \cite{8737464} and \cite{yang2019energy}, the authors sacrificed the training accuracy of FL to improve the convergence time. Moreover, most of these existing works  \cite{zhu2018towards,8733825,8664630,chen2019joint,8737464,8851408,vu2019cell,ren2019accelerating,8851249,feng2018joint,yang2019energy,skatchkovsky2019federated} used random or fixed user selection methods for FL training, which may significantly increase the FL convergence time and also decrease the FL {\color{black}training loss}. In addition, none of these existing works  \cite{zhu2018towards,8733825,8664630,chen2019joint,8737464,8851408,vu2019cell,ren2019accelerating,8851249,feng2018joint,yang2019energy,skatchkovsky2019federated} considers the effect of the local FL models of the users that cannot connect to the BS due to limited wireless resources on the FL convergence time and {\color{black}training loss}.  {\color{black}In \cite{NIPS2018_7752}, the authors proposed a lazily aggregate gradient (LAG) algorithm for optimizing user selection and local FL model transmission. However, the LAG algorithm in \cite{NIPS2018_7752} is a deterministic user selection scheme \cite{NIPS2018_7752} in which some users may not get an opportunity to send their local FL models to the BS thus increasing the training loss}.

%none of these existing works \cite{zhu2018towards,8733825,8664630,chen2019joint} studied the use of artificial neural networks (ANNs) to build the relationship among  different FL model parameters that are transmitted from the devices to the central controller. Building the relationship among various devices' FL model parameters enables the central controller to predict all devices' FL model parameters using only one device's FL model parameters, and thus reducing the number of devices that must transmit FL model parameters to the central controller.   
   
%
{\color{black}
The main contribution of this work is a novel framework for jointly minimizing the FL convergence time and the FL {training loss} as captured by the learning loss function.
 %loss function that captures the FL performance.
  Our
key contributions include:

\begin{itemize}
  
\item We develop a realistic implementation of FL over a wireless network in which the users train their local FL models using their own data and transmit the trained local FL models to a base station (BS) over wireless links. The BS aggregates the received local FL models to generate a global FL model and send it back to the users. {Since the number of resource blocks (RBs) that are used for FL model transmission is limited, the BS must select an appropriate set of users to perform FL at each learning step. Meanwhile, the BS must allocate its RBs to its users so as to improve the convergence speed. 
 %optimize the user selection and RB allocation schemes so as to decrease the number of iterations needed for FL convergence as well as the time duration of each iteration, and optimize the FL performance. 
 Hence, we formulate this joint user selection and RB allocation problem as an optimization problem whose goal is to minimize the number of iterations needed for FL convergence as well as the time duration of each iteration while optimizing the FL {training loss}, in terms of the accuracy.}

\item To solve this problem, we first propose a probabilistic user selection scheme in which users, whose local FL models have large effects on the global FL model, will have high probabilities to connect to the BS. {The reason that optimizing user selection can improve FL convergence speed is that the global model contains less information from the local datasets of some users. In this case, increasing the connectivity probability of these users enables the global model to acquire more information from these users thus improving the FL convergence speed.}
Moreover, the proposed user selection scheme guarantees that every user has a non-zero chance to connect to the BS and, hence, it enables the BS to learn all of the users' training data samples so as to guarantee that the FL algorithm achieves the optimal {training loss}. Given the user selection scheme, the optimal RB allocation scheme can be determined\footnote{{\color{black}Optimizing RB allocation can improve the users' data rates thus decreasing the local FL model transmission delay. In consequence, we can optimize RB allocation to decrease the FL convergence time.}}.

\item To further reduce the FL convergence time, we propose the use of artificial neural networks (ANNs) to find a relationship among the users' local FL model parameters. Building this relationship enables the BS to estimate the local FL model parameters of the users that cannot transmit their local FL models to the BS due to limited number of RBs at each given learning step. Hence, using ANNs, the BS can integrate more users' local FL model parameters to generate the global FL model and, hence, decrease the FL {training loss} and improve convergence speed. 
 
 \item We perform fundamental analysis on the expression of expected
convergence rate of the proposed FL algorithm and we show that, the FL training method (with full gradient descent or stochastic gradient descent \cite{amiri2019machine}), RB allocation, user selection scheme, and the accuracy of predicting local FL models of users
 will significantly affect the convergence speed and {training loss} of FL. 
%Meanwhile, by appropriately setting the learning rate and selecting the number
%of users that perform FL algorithms, the effect of the transmission errors on FL algorithm
%can be reduced and the convergence of FL can be guaranteed.
 
\end{itemize}
Simulation results assess the various performance metrics of the proposed approach and show that the proposed FL
 approach can reduce the convergence time by up to $56\%$ and improve the FL {identification accuracy} by up to $3\%$, compared to a standard FL algorithm. To our best knowledge,  \emph{this is the first work that studies the use of ANNs for the prediction of FL model parameters that are used in the training process so as to improve the FL convergence speed and {training loss}.}}

%To minimize training errors due to wireless links, we formulate a joint resource allocation and user selection problem for FL as an optimization problem whose goal is to minimize the value of the FL loss function while meeting the delay and energy consumption requirements of executing FL. Hence, our framework \emph{jointly considers learning and wireless networking metrics}.
%To solve this problem, we first derive a closed-form expression for the expected convergence rate of the FL algorithm so as to build the relationship between the packet error rates and the performance of the FL algorithm.
%Based on this relationship, the optimization problem can be simplified as an mixed-integer nonlinear programming problem. To solve this simplified problem, we first find the optimal transmit power under given user selection and RB allocation. Then, we transform the original optimization problem into a bipartite matching  problem. Finally, a Hungarian algorithm is used to find the optimal user selection and RB allocation. 
%
%
%
The rest of this paper is organized as follows. The system model and problem formulation are described in Section \uppercase\expandafter{\romannumeral2}. The probabilistic user selection scheme, resource allocation scheme, and the use of ANNs for the prediction of users' local FL models are introduced in Section \uppercase\expandafter{\romannumeral3}. Simulation results are analyzed in Section \uppercase\expandafter{\romannumeral4}. Conclusions are drawn in Section \uppercase\expandafter{\romannumeral5}.

\section{System Model and Problem Formulation}\label{se:system}

\begin{table*}\footnotesize
{\color{black}
  \newcommand{\tabincell}[2]{\begin{tabular}{@{}#1@{}}#2\end{tabular}}
\renewcommand\arraystretch{0.9}
 \caption{
    \vspace*{-0.1em}List of notations}\vspace*{-0.5em} \label{ta:1}
\centering  
\begin{tabular}{|c||c|c||c|}% ±íÊŸž÷ÁÐÔªËØ¶ÔÆë·œÊœ£¬left-l,right-r,center-c
%\hline
%\textbf{Type} & \textbf{Power} & \textbf{Cache} & \textbf{Moving}&\textbf{Path loss} & \textbf{Frequency bands}&\textbf{Interference} \\
\hline
Notation & Description & Notation & Description\\
\hline
$\boldsymbol{g}$ & Global FL model & $K$ & Total number of training data samples of all users \\
\hline
$\boldsymbol{a}_{\mu}$& User selection vector at iteration $\mu$  & $K_i$ & Number of training data samples of each user $i$ \\
\hline
 $\lambda$ & Learning rate& $p_{i,\mu}$ & Probability of user $i$ connecting to the BS at iteration $\mu$\\
\hline
$R$ & Number of RBs &  $\boldsymbol{w}_{i,\mu}$ & Local FL model of user $i$ at iteration $\mu$ \\
\hline
 $\gamma$ & Prediction error requirement & $\boldsymbol{r}_{i,\mu}$ & RB allocation vector of user $i$ at iteration $\mu$ \\
\hline
$B$ & Bandwidth of each RB & $E_{i,\mu}$ & Prediction error of user $i$'s local FL model at iteration $\mu$ \\
\hline
 $\gamma_R$ & Number of users considered in (10)&$\boldsymbol{R}_{\mu}$& RB allocation matrix at iteration $\mu$ \\
\hline
$c_{i}^\textrm{D}$ & Downlink data rate of each user $i$ & $\boldsymbol{R}$& RB allocation matrix from iteration 1 to $T$\\
\hline
 $I_n$ & Interference over RB $n$&$\boldsymbol{A}$ & User selection matrix from iteration 1 to $T$\\
\hline
$B^\textrm{D}$ & Downlink bandwidth &$\boldsymbol{\hat w}_{i,\mu}$ & Predicted local FL model of user $i$ at iteration $\mu$ \\
\hline
  $Z$& Data size of local or global FL model & ${i}^*$& The user always connects to the BS  \\
\hline
$\Omega_\mu $ & Convergence indicator at iteration $\mu$ &  $T$ & Maximum number of iterations \\
\hline
%$x_{t,i},y_{t,i}$ & Coordinates of users   & $P_B$ & Transmit power of the BBUs\\
%\hline
\end{tabular}}
\end{table*}  

%
%%
%%\begin{figure}[!t]
%%  \begin{center}
%%   \vspace{0cm}
%%    \includegraphics[width=15cm]{12313.eps}
%%    \vspace{-0.3cm}
%%    \caption{\label{figure1b} The components of one VR view. Here, the $360^\circ$ image for each eye is slightly different. These different images generate a 3D image for each VR user.}
%%  \end{center}\vspace{-0.3cm}
%%\end{figure}
%%

%\begin{figure}[!t]
%  \begin{center}
%   \vspace{0cm}
%    \includegraphics[width=8cm]{FLarchitecture.eps}
%    \vspace{-0.3cm}
%    \caption{\label{model1} The architecture of a wireless network that performs an FL algorithm.}
%  \end{center}\vspace{-0.8cm}
%\end{figure}

%%
%%
%%\begin{figure}[!t]
%%  \begin{center}
%%   \vspace{0cm}
%%    \includegraphics[width=3.9cm]{FLprocedure.eps}
%%    \vspace{-0.2cm}
%%    \caption{\label{FLprocedure} The learning procedure of an FL algorithm.}
%%  \end{center}\vspace{-0.8cm}
%%\end{figure}
%
%
 \begin{figure}[!t]
  \begin{center}
   \vspace{0cm}
    \includegraphics[width=8cm]{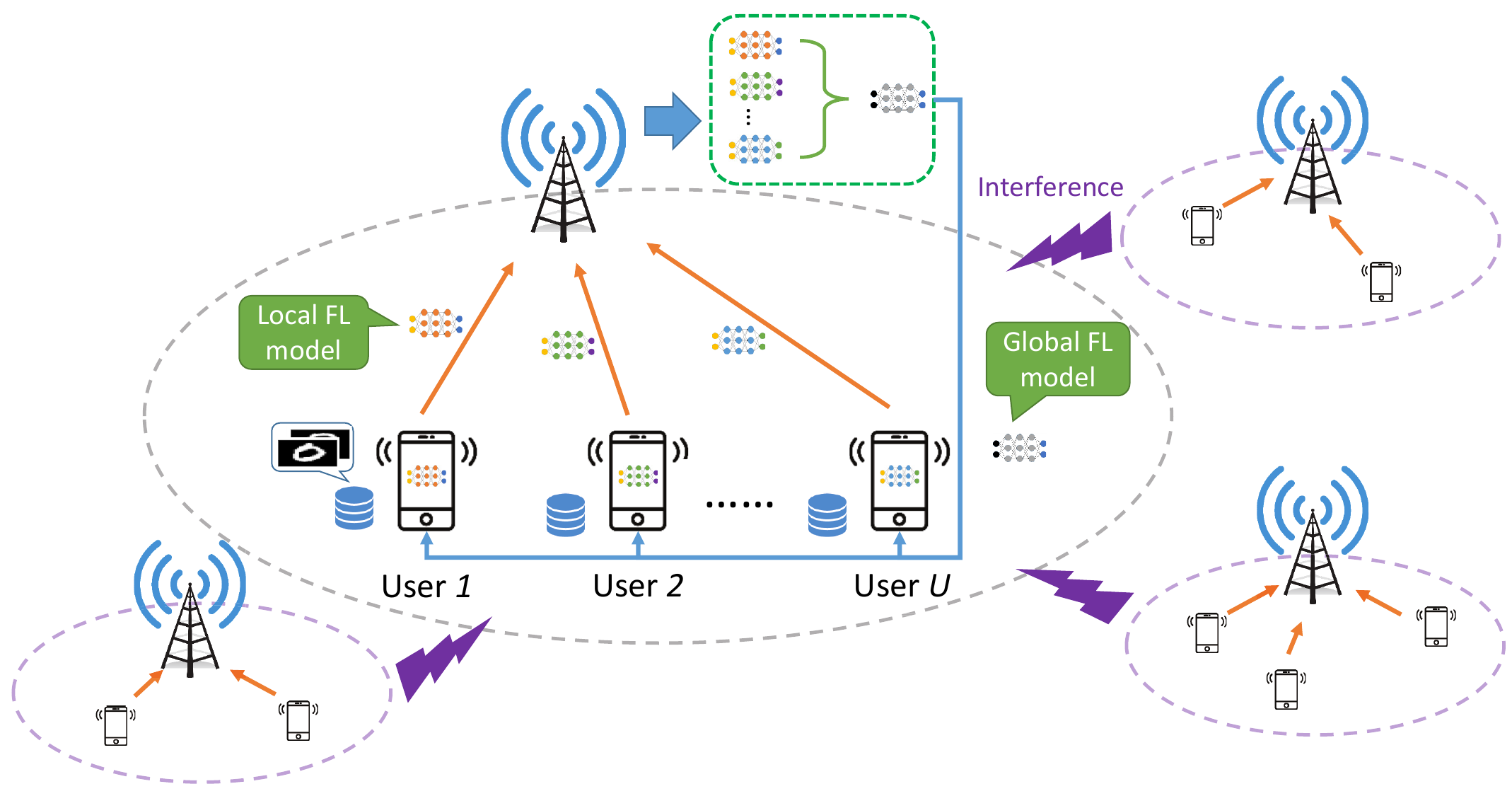}
    \vspace{-0.25cm}
 {\color{black}   \caption{\label{architecture}The architecture of an FL algorithm that is being executed over a wireless network with multiple devices and a single base station.}}
  \end{center}\vspace{-0.1cm}
\end{figure}

Consider a wireless network in which a set $\mathcal{U}$ of $U$ users and one BS jointly execute an FL algorithm for data analysis and inference, {\color{black}as shown in Fig. \ref{architecture}}. We assume that each user $i$ collects $K_i$ training data samples and each training data sample $k$ consists of an input $\boldsymbol{x}_{ik} \in \mathbb{R}^{N_\textrm{in}\times1} $ and its corresponding output $\boldsymbol{y}_{ik} \in \mathbb{R}^{N_\textrm{out}\times1}$.
We also assume that the data collected by the users follows the same distribution \cite{konevcny2016federated} or statistical heterogeneity \cite{sahu2018convergence}.
The FL training process is done in a way to solve
\begin{equation}\label{eq:ML}
\mathop {\min } \limits_{{\boldsymbol{g}}} \frac{1}{K}  \sum\limits_{i = 1}^U  \sum\limits_{k=1 }^{K_i} {f\left( {{\boldsymbol{g}},{\boldsymbol{x}_{ik}},{\boldsymbol{y}_{ik}}} \right)} ,
\end{equation}
where $K=\sum\limits_{i = 1}^U K_i$ is total number of training data samples of all users, $\boldsymbol{g}$ is a vector that captures the FL model trained by $K$ training data samples, and ${f\left( {{\boldsymbol{g}},{\boldsymbol{x}_{ik}},{\boldsymbol{y}_{ik}}} \right)}$ is a loss function that captures the accuracy of the considered FL algorithm by building a relationship between an input vector $\boldsymbol{x}_{ik}$ and an output vector $\boldsymbol{y}_{ik}$. For different learning tasks, the definition of the loss functions can be different \cite{8755300}. 
In particular, for our considered handwritten digit {\color{black}classification} task, the {\color{black}cross entropy} is ${f\left( {{\boldsymbol{g}},{\boldsymbol{x}_{ik}},{\boldsymbol{y}_{ik}}} \right)}=-\boldsymbol{y}_{ik} \log\left( \boldsymbol{x}_{ik}^{\rm{T}}{\boldsymbol{g}}  \right)+\left(1-\boldsymbol{y}_{ik}\right) \log\left( 1-\boldsymbol{x}_{ik}^{\rm{T}}{\boldsymbol{g}}  \right)$. To find the optimal vector $\boldsymbol{g}$, conventional centralized learning algortihms require the users to share their collected data with other users, which is impractical due to communication overhead and data privacy issues. To address these issues, FL can be used. The training procedure of FL proceeds as follows \cite{bonawitz2019towards}:    
 \begin{itemize}
\item[a.] The BS broadcasts the information that is used to initialize the learning model to each user. Then, each user initializes its learning model. 
\item[b.] Each user trains its generated machine learning model using its collected data and {\color{black}sends the trained learning model parameters to the BS.}
\item[c.] The BS integrates the received learning model parameters and broadcasts them back to all users.
\item[d.] Steps b. and c. are repeated until the optimal vector $\boldsymbol{g}$ is found.    
\end{itemize} 
From this FL training procedure, we observe that the users do not have to share their collected data with other users and the BS. In contrast, they only need to share their trained learning model parameters with the BS. Note that the implementation of steps b. and c. constitutes one iteration of training FL. Hereinafter, the FL model that is trained by each user using its collected data is called \emph{local FL model} while the FL model that is generated by the BS is called \emph{global FL model}. Since all of the FL model parameters are transmitted over wireless networks, we must consider the effect of {\color{black}wireless factors such as resource blocks, user selection, dynamic wireless channel, and transmission delay,} on the FL {\color{black}training loss and convergence time}.
Next, we first mathematically formulate the training procedure of FL. Then, the transmission of FL model parameters over wireless links is introduced. Finally, the problem of minimzing FL convergence time and {\color{black}training loss} is formulated. {\color{black}Table \uppercase\expandafter{\romannumeral1} provides a summary of the notations used hereinafter.}

\subsection{Training Procedure of Federated Learning}
The local FL model of each user $i$ at each iteration $\mu$ is defined as a vector $\boldsymbol{w}_{i,\mu} \in \mathbb{R}^{W\times1}$, where $W$ is the number of elements in $\boldsymbol{w}_{i,\mu}$. 
%The global model at each iteration $\mu$ is $\boldsymbol{g}_\mu$. 
The update of the global FL model at iteration $\mu$ is given by \cite{45648}
\begin{equation}\label{eq:w}
\setlength{\abovedisplayskip}{2 pt}
\setlength{\belowdisplayskip}{2 pt}
\boldsymbol{g}\left(\boldsymbol{a}_{\mu} \right)=\sum\limits_{i = 1}^U{\boldsymbol{w}_{i,\mu} \left( \frac{a_{i,\mu} K_i}{\sum\limits_{i = 1}^U a_{i,\mu}K_i}\right) },
\end{equation}   
where $\boldsymbol{g}\left(\boldsymbol{a}_{\mu} \right)$ is the global FL model at iteration $\mu$, $\frac{a_{i,\mu} K_i}{\sum\limits_{i = 1}^U a_{i,\mu}K_i}$ is a scaling update weight of $\boldsymbol{w}_{i,\mu}$ and $\boldsymbol{a}_\mu=\left[a_{1,\mu},\ldots,a_{U,\mu}\right]$ is a user association vector with $a_{i,\mu} \in \left\{0,1 \right\}$, $a_{i,\mu}=1$ indicates that user $i$ connects to the BS and user $i$ must send its local FL model $\boldsymbol{w}_{i,\mu}$ to the BS at iteration $\mu$, otherwise, we have $a_{i,\mu}=0$. From (\ref{eq:w}), we see that the user association vector $\boldsymbol{a}_\mu$ will change at each iteration. This is due to the fact that, in wireless networks, the number of RBs is limited and the BS must select a suitable set of users to transmit their local FL models in each iteration.

The update of the local FL model $\boldsymbol{w}_{i,\mu}$ depends on the used training method. In our model, the full gradient descent method  \cite{konevcny2016federated} is used to update the local FL model. In Section~\ref{se:CIC}, we revisit our analysis when FL is executed with a stochastic gradient descent method to update the local FL model. Using the full gradient descent method, the update of local FL model can be given by
\begin{equation}\label{eq:g}
\boldsymbol{w}_{i,\mu+1}=\boldsymbol{ g}\left(\boldsymbol{a}_{\mu}\right) -\frac{\lambda}{K_i} \sum\limits_{k =1 }^{K_i} \nabla {f\left( { \boldsymbol{ g}\left(\boldsymbol{a}_\mu\right) ,{\boldsymbol{x}_{ik}},{\boldsymbol{y}_{ik}}} \right)},
\end{equation}  
where $\lambda$ is the learning rate and $\nabla {f\left( { \boldsymbol{ g}\left(\boldsymbol{a}_\mu\right) ,{\boldsymbol{x}_{ik}},{\boldsymbol{y}_{ik}}} \right)}$ is the gradient of ${f\left( {\boldsymbol{ g}\left(\boldsymbol{a}_\mu\right) ,{\boldsymbol{x}_{ik}},{\boldsymbol{y}_{ik}}} \right)}$ with respect to $\boldsymbol{ g}\left(\boldsymbol{a}_\mu\right) $.  Based on (\ref{eq:w}) and (\ref{eq:g}), the BS and the users can update their FL models so as to find the optimal global FL model that solves problem (\ref{eq:ML}).

%Let $\boldsymbol{{w}}_{i,h}$ be user $i$'s local FL model that the BS receives at iteration $h\in \left[1,2,\ldots, H \right]$ where $H$ is the total number of iterations that the BS and the users need for convergence. Similarly, $\boldsymbol{{w}}^{h}$ denotes the global FL model that the BS transmits to the users at iteration $h$. 
%The update of each user $i$'s local FL model and the global FL model of the BS can be given by:
%\begin{equation}
%\boldsymbol{{w}}_{i,h}
%\end{equation}
%During the training process, each user will first use its training data $\boldsymbol{X}_i$ and $\boldsymbol{y}_i$ to train the local FL model $\boldsymbol{w}_i$ and then, transmit $\boldsymbol{w}_i$ to the BS via wireless cellular links. Once the BS receives the local FL models, it will update the global FL model based on (\ref{eq:w}) and transmit the global FL model $\boldsymbol{g}$ to all users to optimize the local FL models. As time elapses, the BS and users can find their optimal FL models and use them to minimize the loss function in (\ref{eq:ML}).

\subsection{Transmission Model}
We assume that an orthogonal frequency-division multiple access (OFDMA) technique is adopted for local FL model transmission from the users to the BS. 
In our model, we assume that the total number of RBs that the BS can allocate to its users is $R$, and each user can occupy only one RB.
The uplink rate of user $i$ that is transmitting its local FL model parameters to the BS at iteration $\mu$ is given by{\color{black}\cite{6497017}}
%\begin{equation}\label{eq:downlinkdatarate}
%c_{i}=   \sum\limits_{k=1}^S s_{ik} {\log _2}\left(\!1\!+\! {\frac{{{P_{i}}{ h_{ij}}}}{{{{{{{B\sigma ^2}  }}} }}}} \!\right),
%\end{equation}
\begin{equation}\label{eq:uplinkdatarate}
c_{i}^\textrm{U} \left(\boldsymbol{r}_{i,\mu}\right)=  \sum\limits_{n = 1}^R r_{in,\mu} B{\log _2}\left(\!1\!+\! {\frac{{{P}{h_{i}}}}{I_n+BN_0}} \!\right),
\end{equation}
 where $\boldsymbol{r}_{i,\mu}=\left[{r}_{i1,\mu},\ldots, {r}_{iR,\mu}\right]$ is an RB allocation vector with $r_{in,\mu} \in \left\{0,1\right\}$; $r_{in,\mu}=1$ indicates that RB $n$ is allocated to user $i$ at iteration $\mu$, otherwise, we have $r_{in,\mu}=0$, and
 $P$ is the transmit power of user $i$, which is assumed to be equal for all users. %$ h_{i}=o_{i}d_{i}^{-\alpha}$
$h_i$ is the channel gain between user $i$ and the BS, %where $o_{i}$ is the Rayleigh fading power gain; $d_{i}$ is the distance between user $i$ and the BS; $\alpha$ is the path loss exponent; 
$N_0$ is the noise power spectral density, and $I_n$ is the interference caused by the users that connect to other BSs using the same RB. Since the interference caused by the users using the same RBs may significantly affect the transmission delay and the time needed for FL to converge, we must consider this interference in (\ref{eq:uplinkdatarate}).

 Similarly, the downlink data rate of the BS when transmitting the global FL model parameters to each user $i$ is given by{\color{black}\cite{6497017}}
 \begin{equation}\label{eq:downlinkdatarate}
 \setlength{\abovedisplayskip}{2 pt}
\setlength{\belowdisplayskip}{2 pt}
c_{i}^\textrm{D} =  B^\textrm{D}{\log _2}\left(\!1\!+\! {\frac{{{P_{B}}{ h_{i}}}}{ B^\textrm{D}N_0}} \!\right),
\end{equation}
where $B^\textrm{D}$ is the bandwidth that the BS used to broadcast the global FL model to each user $i$ and $P_B$ is the transmit power of the BS.

Since the number of elements in the local FL model $\boldsymbol{w}_{i,\mu}$ is similar to that of the global FL model $\boldsymbol{g}\left( \boldsymbol{a}_\mu \right)$, the data size of the local FL model $\boldsymbol{w}_{i,\mu}$ is equal to the data size of the global FL model $\boldsymbol{g}\left( \boldsymbol{a}_\mu\right)$. Here, the data size of the local FL model is defined as the number of bits that the user requires to transmit the local FL model vector $\boldsymbol{w}_{i,\mu}$ to the BS.  Let $Z$ be the data size of a global FL model or local FL model.
The transmission delay between user $i$ and the BS over both uplink and downlink at iteration $\mu$ will be
\begin{equation}
l_{i}^\textrm{U}\left(\boldsymbol{r}_{i,\mu}\right) =\frac{Z }{c_{i}^\textrm{U} \left(\boldsymbol{r}_{i,\mu}\right)    },%~l_{i}^\textrm{D}=\frac{Z}{c_{i}^\textrm{D}}.
\end{equation}
\begin{equation}
l_{i}^\textrm{D}=\frac{Z}{c_{i}^\textrm{D}}.
\end{equation}
The time that the users and the BS require to jointly complete an update of their respective local and global FL models at iteration $\mu$ is given by
\begin{equation}\label{eq:tmu}
t_\mu\left(\boldsymbol{a_{\mu}}, \boldsymbol{R}_\mu\right) =\mathop {\max }\limits_{i \in \mathcal{U}}   {{a_{i,\mu}}\left(l_{i}^\textrm{U}\left(\boldsymbol{r}_{i,\mu}\right) +l_{i}^\textrm{D} \right)},
\end{equation}
where $\boldsymbol{R}_\mu=\left[\boldsymbol{r}_{1,\mu},\ldots,\boldsymbol{r}_{U,\mu} \right]$. When $a_{i,\mu}=0$, $ {{a_{i,\mu}}\left(l_{i}^\textrm{U}\left(\boldsymbol{r}_{i,\mu}\right) +l_{i}^\textrm{D} \right)}=0$. Here, $a_{i,\mu}=0$ implies that user $i$ will not send its local FL model to the BS at iteration $\mu$, and hence, user $i$ will not cause any delay at iteration $\mu$, ($ {{a_{i,\mu}}\left(l_{i}^\textrm{U}\left(\boldsymbol{r}_{i,\mu}\right) +l_{i}^\textrm{D} \right)}=0$). When $a_{i,\mu}$=1, then user $i$ will transmit its local FL model to the BS and the transmission delay will be $ {l_{i}^\textrm{U}\left(\boldsymbol{r}_{i,\mu}\right) +l_{i}^\textrm{D}}$. Hence, (\ref{eq:tmu}) is essentially the worst-case transmission delay among all selected users.

\subsection{Problem Formulation}
Having defined the system model, the next step is to introduce a joint RB allocation and user selection scheme to minimize the time that the users and the BS need in order to complete the FL training process. This optimization problem is formulated as follows:
\addtocounter{equation}{0}
\begin{equation}\label{eq:max}
\begin{split}
\mathop {\min }\limits_{\boldsymbol{A}, \boldsymbol{R}}  \sum\limits_{\mu=1}^{T} t_\mu\left(\boldsymbol{a_{\mu}}, \boldsymbol{R}_\mu\right)\Omega_\mu
\end{split}
\end{equation}
\vspace{-0.3cm}
\begin{align}\label{c1}
\setlength{\abovedisplayskip}{-20 pt}
\setlength{\belowdisplayskip}{-20 pt}
&\!\!\!\!\!\!\!\!\rm{s.\;t.}\;\scalebox{1}{$a_{i,\mu}, r_{in,\mu}, \Omega_\mu \in \left\{0,1\right\}, \;\;\;\;\;\forall i \in \mathcal{U}, n=1,\ldots, R,$}\tag{\theequation a}\\
&\scalebox{1}{$\;\;\; \sum\limits_{i \in \mathcal{U}} {r_{i,n}}  \le 1,~\forall n=1,\ldots, R, $} \tag{\theequation b}\\
&\scalebox{1}{$\;\;\; \sum\limits_{n = 1}^R r_{in,\mu}=a_{i,\mu},\;\;\forall i \in \mathcal{U}, $} \tag{\theequation c}
\end{align}
where $\boldsymbol{A}=\left[\boldsymbol{a}_1,\ldots, \boldsymbol{a}_\mu, \ldots, \boldsymbol{a}_T\right]$ is a user selection matrix for all iterations, $\boldsymbol{R}=\left[\boldsymbol{R}_{1},\ldots, \boldsymbol{R}_{\mu}, \ldots, \boldsymbol{R}_{T} \right]$ is an RB allocation matrix for all users at all iterations, and $T$ is a constant, which is large enough to guarantee the convergence of FL. In other words, the number of iterations that the FL algorithm requires to converge will not be larger than $T$. In (\ref{eq:max}), %$\boldsymbol{g}^*$ is the optimal global FL model at convergence, and, 
$\Omega_\mu=1$ implies that the FL algorithm does not converge, otherwise, we have $\Omega_\mu=0$, (\ref{eq:max}a) and (\ref{eq:max}b) imply that each user can only occupy one RB, and (\ref{eq:max}c) implies that all RBs must be allocated to the users that are associated with the BS. From (\ref{eq:max}), we see that the time used for the update of the local and global FL models, $t_\mu$, depends on the user selection vector $\boldsymbol{a}_\mu$ and RB allocation matrix $\boldsymbol{R}_{\mu}$. Meanwhile, as shown in (\ref{eq:g}), the total number of iterations that the FL algorithm needs in order to converge depends on the user selection vector $\boldsymbol{a}_\mu$. In consequence, the time duration of each FL training iteration and the number of iterations needed for the FL algorithm to converge are dependent. Moreover, given the global model $\boldsymbol{g}\left(\boldsymbol{a}_\mu\right)$ at iteration $\mu$, the BS cannot calculate the number of iterations that the FL algorithm needs to converge since all of the training data samples are located at the users' devices.
Hence, problem (\ref{eq:max}) is challenging to solve.

\section{Minimization of Convergence Time of Federated Learning}\label{se:algorithm}
To solve problem (\ref{eq:max}), we first need to determine the user association at each iteration. Given the user selection vector, the optimal RB allocation scheme can be derived. To further improve the convergence speed of FL, ANNs are introduced to estimate the local FL models of the users that are not allocated any RBs for transmitting their local model parameters at each given learning step. Fig. \ref{flow} summarizes the training procedure of the proposed FL.

\begin{figure}[!t]
  \begin{center}
   \vspace{0cm}
    \includegraphics[width=8cm]{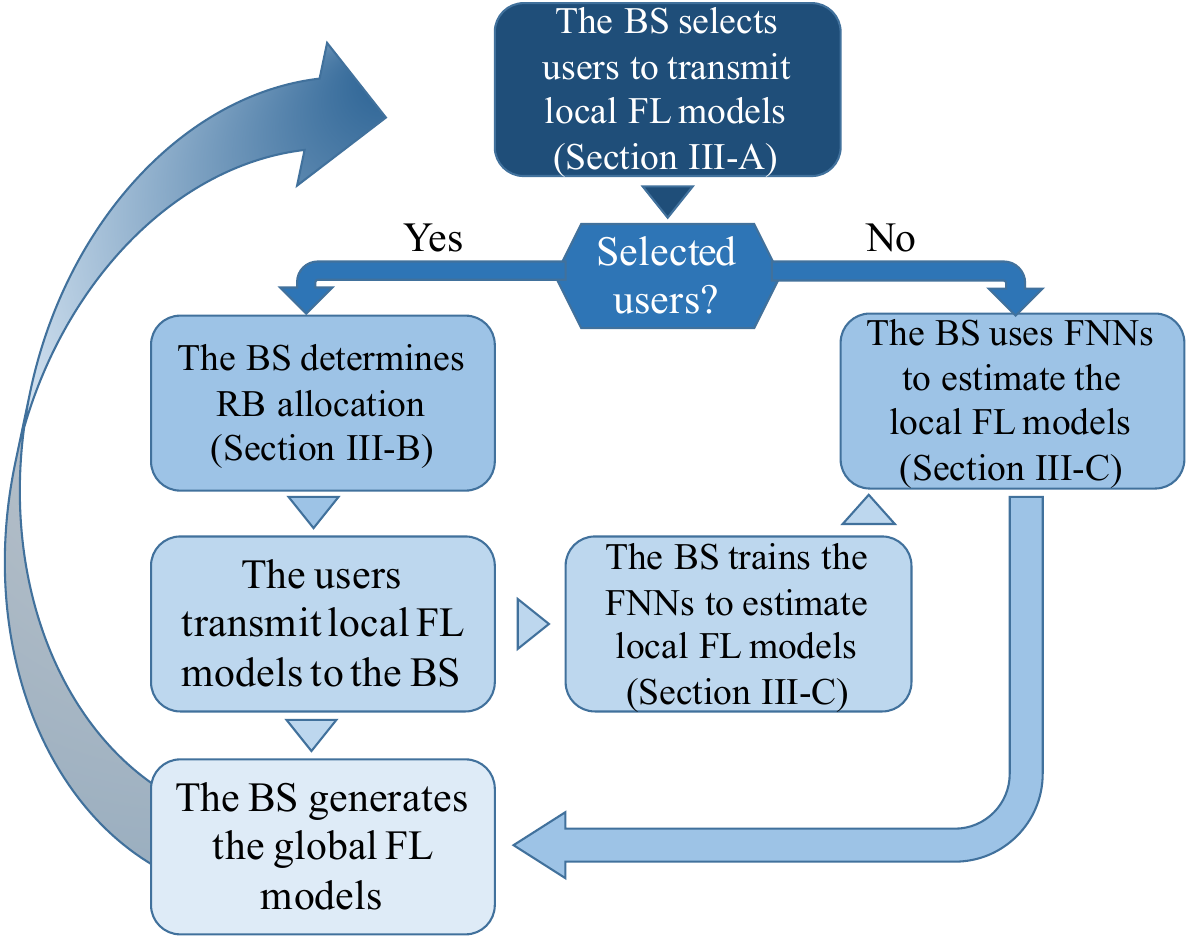}
    \vspace{-0.25cm}
    \caption{\label{flow} The training procedure of the proposed FL.}
  \end{center}\vspace{-0.5cm}
\end{figure}

\subsection{Gradient Based User Association Scheme}\label{se:userassociation}
To predict the local FL model of each user, the BS needs to use the local FL model of a given user as an ANN input, as will be explained in Subsection \ref{se:prediction}. Hence, in the proposed user association scheme, one user must be selected to connect with the BS during all training iterations. 
%In particular, this user will transmit the local FL models to the BS at each iteration so as to provide input information for the ANNs.
 To determine the user that connects to the BS during the entire training process, we first assume that the distance $d_i$ between user $i$ and the BS satisfies $d_1\le d_2 \le \ldots \le d_U$. Hence, user $i^*$ that always connects to the BS can be found from
\begin{equation}\label{eq:i}
i^*=\arg\mathop {\max }\limits_{i \in \mathcal{U} }  \sum\limits_{k =1 }^{K_i} \nabla  {f\left( { \boldsymbol{ g}\left(\boldsymbol{a}_{\mu-1}\right) ,{\boldsymbol{x}_{ik}},{\boldsymbol{y}_{ik}}} \right)},
\end{equation}
\begin{align}\label{c1}
\setlength{\abovedisplayskip}{-20 pt}
\setlength{\belowdisplayskip}{-20 pt}
&\!\!\!\!\!\!\!\!\rm{s.\;t.}\;\scalebox{1}{$d_1\le d_i \le d_{\gamma_R}, \forall i \in \mathcal{U}, $}\tag{\theequation a}
\end{align}
where {\color{black} $\gamma_R$ is the number of users considered in (\ref{eq:i}) with $1 \le \gamma_R \le U$.} As $\gamma_R$ increases, the number of users considered in (\ref{eq:i}) increases. Hence, the transmission delay of user $i^*$ may increase thus increasing the time used to complete one FL iteration. However, as $\gamma_R$ increases, the value of $\mathop {\max }\limits_{i \in \mathcal{U} }  \sum\limits_{k =1 }^{K_i} \nabla  {f\left( { \boldsymbol{ g}\left(\boldsymbol{a}_{\mu-1}\right) ,{\boldsymbol{x}_{ik}},{\boldsymbol{y}_{ik}}} \right)}$ may increase, and, thus, the number of iterations required for FL to converge decreases. Here, user $i^*$ is determined at the first iteration.

%In our algorithm, we select the user that is closet to the BS. In other words, if the distance between user $i^*$ and the BS is smaller than the distance between the BS and other users, we have $a_{i^*,\mu}=1, \mu=1,\ldots,T$. 
From (\ref{eq:g}), we can observe that, at each iteration $\mu$, the global FL model $\boldsymbol{g}\left(\boldsymbol{a}_{\mu-1}\right)$ will change ${\lambda} \sum\limits_{k =1 }^{K_i} \nabla {f\left( { \boldsymbol{ g}\left(\boldsymbol{a}_{\mu-1}\right) ,{\boldsymbol{x}_{ik}},{\boldsymbol{y}_{ik}}} \right)}$ due to the local FL model of a given user $i$. We define a vector $\boldsymbol{e}_{i,\mu}={\lambda} \sum\limits_{k =1 }^{K_i} \nabla {f\left( { \boldsymbol{ g}\left(\boldsymbol{a}_{\mu-1}\right) ,{\boldsymbol{x}_{ik}},{\boldsymbol{y}_{ik}}} \right)}$ as the change in the global FL model due to user $i$'s local FL model. To enable the BS to predict each user's local FL model at each learning step, each user must have a chance to connect to the BS so as to provide the training data samples (local FL model parameters) to the BS for training ANNs. Therefore, a probabilistic user association scheme is developed, which is given by
\begin{equation}\label{eq:probability}
{p_{i,\mu}} = \left\{ {\begin{array}{*{20}{c}}
  \frac{\left\|\boldsymbol{e}_{i,\mu}\right\|}{\sum\limits_{i=1,i\neq i^*}^{U}\left\|\boldsymbol{e}_{i,\mu}\right\|},~\textrm{if}~i \neq i^*, \\ 
\;\;\;\; \;\;\;\;\;\;\; \;\;\;1,\;\;\;\;\;\;\;\;~\textrm{if}~i=i^*,\;\;\;\;\;\;  
\end{array}} \right.
\end{equation}
where $p_{i,\mu}$ represents the probability that user $i$ connects to the BS at iteration $\mu$, and $\left\|\boldsymbol{e}_{i,\mu}\right\|$ is the norm of vector $\boldsymbol{e}_{i,\mu}$. From (\ref{eq:probability}),  we can see that, as $\left\|\boldsymbol{e}_{i,\mu}\right\|$ increases, the probability of associating user $i$ with the BS increases. In consequence, the probability that the BS uses user $i$'s local FL model to generate the global FL model increases. Hence, using the proposed user association scheme in (\ref{eq:probability}), the BS has a high probability to connect to the user whose local FL model significantly affects the global FL model, thus improving the FL convergence speed.
From (\ref{eq:probability}), we also see that user $i^*$ will always connect to the BS so as to provide information for the prediction of other users' local FL models. Based on (\ref{eq:probability}), the user association scheme at each iteration can be determined. To calculate ${p_{i,\mu}}$ in (\ref{eq:probability}), the BS only needs to know $\left\|\boldsymbol{e}_{i,\mu}\right\|$ of each user $i$ without requiring the exact training data information. In fact, $\left\|\boldsymbol{e}_{i,\mu}\right\|$ can be directly calculated by user $i$ and each user $i$ needs to transmit only a scalar $\left\|\boldsymbol{e}_{i,\mu}\right\|$ to the BS.

\subsection{Optimal RB Allocation Scheme}  
%Since the user association is updated at each iteration, the RB allocation must be updated according to the user association. 
Given the user association scheme at each iteration $\mu$, problem (\ref{eq:max}) at iteration $\mu$ can be simplified as follows:
\addtocounter{equation}{0}
\begin{equation}\label{eq:RB}
\begin{split}
\mathop {\min }\limits_{\boldsymbol{R}_{\mu}} t_\mu \left( \boldsymbol{R}_\mu\right)=\mathop {\min }\limits_{\boldsymbol{R}_{\mu}}\mathop {\max }\limits_{i \in \mathcal{U}}   {{a_{i,\mu}}\left( l_{i}^\textrm{U}\left(\boldsymbol{r}_{i,\mu}\right) +l_{i}^\textrm{D} \right)}
\end{split}
\end{equation}
\vspace{-0.3cm}
\begin{align}\label{c1}
\setlength{\abovedisplayskip}{-20 pt}
\setlength{\belowdisplayskip}{-20 pt}
&\!\!\!\!\!\!\!\!\rm{s.\;t.}\;\scalebox{1}{$r_{in,\mu} \in \left\{0,1\right\}, \;\;\;\;\;\forall i \in \mathcal{U}, n=1,\ldots, R,$}\tag{\theequation a}\\
&\scalebox{1}{$\;\;\; \sum\limits_{i \in \mathcal{U}} {r_{in,\mu}}  \le 1,~\forall n=1,\ldots, R, $} \tag{\theequation b}\\
&\scalebox{1}{$\;\;\; \sum\limits_{n = 1}^R r_{in,\mu}=a_{i,\mu},\;\;\forall i \in \mathcal{U}. $} \tag{\theequation c}
\end{align}
We assume that there exists a variable $m$ that satisfies ${{a_{i,\mu}}\left(l_{i}^\textrm{U}\left(\boldsymbol{r}_{i,\mu}\right) +l_{i}^\textrm{D} \right)}  \leqslant m $. Problem (\ref{eq:RB}) can be simplified as follows:
\addtocounter{equation}{0}
\begin{equation}\label{eq:RB1}
\begin{split}
\mathop {\min }\limits_{\boldsymbol{R}_{\mu}}m
\end{split}
\end{equation}
\vspace{-0.3cm}
\begin{align}\label{c1}
\setlength{\abovedisplayskip}{-20 pt}
\setlength{\belowdisplayskip}{-20 pt}
&\!\!\!\!\!\!\!\!\rm{s.\;t.}\;\scalebox{1}{(\ref{eq:RB}a), (\ref{eq:RB}b), and (\ref{eq:RB}c),}\tag{\theequation a}\\
&\scalebox{1}{$\;\;\;m \geqslant  {{a_{i,\mu}}\left(l_{i}^\textrm{U}\left(\boldsymbol{r}_{i,\mu}\right) +l_{i}^\textrm{D} \right)}, \forall i \in \mathcal{U}. $} \tag{\theequation b}
\end{align}
Since (\ref{eq:RB1}b) is nonlinear, we must transform it into a linear constraint. We first assume that $l_{in,\mu}^\textrm{U}=\frac{Z }{B{\log _2}\left(\!1\!+\! {\frac{{{P}{h_{i}}}}{I_n+BN_0}} \!\right)}$, which represents the delay of user $i$ transmitting the local FL model over RB $n$ at iteration $\mu$. Then, we have $l_{i}^\textrm{U}\left(\boldsymbol{r}_{i,\mu}\right)=\sum\limits_{n=1}^R r_{in,\mu}l_{in,\mu}^\textrm{U}$. Hence, problem (\ref{eq:RB1}) can be rewritten as follows:
\addtocounter{equation}{0}
\begin{equation}\label{eq:RB2}
\begin{split}
\mathop {\min }\limits_{\boldsymbol{R}_{\mu}}m
\end{split}
\end{equation}
\vspace{-0.3cm}
\begin{align}\label{c1}
\setlength{\abovedisplayskip}{-20 pt}
\setlength{\belowdisplayskip}{-20 pt}
&\!\!\!\!\!\!\!\!\rm{s.\;t.}\;\scalebox{1}{(\ref{eq:RB}a), (\ref{eq:RB}b), and (\ref{eq:RB}c),}\tag{\theequation a}\\
&\scalebox{1}{$\;\;\;m \geqslant  {{a_{i,\mu}}\left(\sum\limits_{n=1}^R r_{in,\mu}l_{in,\mu}^\textrm{U} +l_{i}^\textrm{D} \right)}, \forall i \in \mathcal{U}. $} \tag{\theequation b}
\end{align}
Problem (\ref{eq:RB2}) is equivalent to (\ref{eq:RB1}) and is an integer linear programming problem, which can be solved by known optimization algorithms such as interior-point methods \cite{10556789808805699}. 

\subsection{Prediction of the Local FL Models}\label{se:prediction}
The previous subsections determine the users that are associated with the BS at each iteration and minimize their transmission delay by optimizing the RB allocation. Next,
we introduce an ANN-based algorithm to predict the local FL model parameters of the users that are not allocated any RBs for local FL model transmission at each given learning step. In particular, ANNs are used to build a relationship between the local FL models of different users. Since {\color{black}fully-connected multilayer perceptron (MLP) in ANNs} are good at function fitting tasks and finding a relationship among different users' local FL models is a {\color{black}regression} task, we prefer to use {\color{black}MLP} instead of other neural networks such as recurrent neural networks. Next, we first introduce the architecture of our FNN-based algorithm. Then, we explain how to implement this algorithm to predict the local FL model parameters of the users at each given learning step.
  
 Our FNN-based prediction algorithm consists of three components: a) input, b) a single hidden layer, and c) output, which will be defined as follows:
  \begin{itemize}
\item \textbf{Input}: The input of the FNN that is used for the prediction of user $j$'s local FL model is a vector $\boldsymbol{w}_{i^*,\mu}$, which represents the local FL model of user $i^*$. As we mentioned in Subsection \ref{se:userassociation}, user $i^*$ will always connect with the BS so as to provide the input information for the {\color{black}MLP} to predict the local FL models of other users.  
\item \textbf{Output}: The output of the FNN for the prediction of user $j$'s local FL model is a vector $\boldsymbol{o}=\boldsymbol{w}_{i^*,\mu}-\boldsymbol{w}_{j,\mu}$, which represents the difference between user $i^*$'s local FL model and user $j$'s local FL model. Based on the prediction output $\boldsymbol{o}$ and user $i^*$'s local FL model, we can obtain the local FL model of user $j$, i.e., $\boldsymbol{\hat w}_{j,\mu}=\boldsymbol{w}_{i^*,\mu}- \boldsymbol{o}$ with $\boldsymbol{\hat w}_{j,\mu}$ being the predicted user $j$'s local FL model.

\item \textbf{A single hidden layer}: The hidden layer of an FNN allows it to learn nonlinear relationships between input vector $\boldsymbol{w}_{i^*,\mu}$ and output vector $\boldsymbol{o}$. Mathematically, a single hidden layer consists of $N$ neurons. The weight matrix that represents the connection strength between the input vector and the neurons in the hidden layer is $\boldsymbol{v}^\textrm{in} \in \mathbb{R}^{N\times W}$. Meanwhile, the weight matrix that captures the strengths of the connections between the neurons in the hidden layer and the output vector is $\boldsymbol{v}^\textrm{out} \in \mathbb{N}^{W\times N}$.     
\end{itemize}

Given the components of the {\color{black}MLP}, next, we introduce the use of {\color{black}MLP} to predict each user's local FL model. The states of the neurons in the hidden layer are given by
\begin{equation}
\boldsymbol{\vartheta}=\sigma\left(\boldsymbol{v}^\textrm{in}\boldsymbol{w}_{i^*,\mu}+\boldsymbol{b}_\vartheta \right),
\end{equation}
where $\sigma\left(\boldsymbol{x}\right)=\frac{2}{{1 + \exp \left( { - 2\boldsymbol{x}} \right)}} - 1$ and $\boldsymbol{b}_\vartheta \in \mathbb{R}^{N \times 1}$ is the bias. Given the neuron states, we can obtain the output of the FNN, as follows:
\begin{equation}\label{eq:output}
\boldsymbol{o}=\boldsymbol{v}^\textrm{out}\boldsymbol{\vartheta}+\boldsymbol{b}_o,
\end{equation}
where $\boldsymbol{b}_o  \in \mathbb{R}^{W \times 1} $ is a vector of bias. Based on (\ref{eq:output}), we can calculate the predicted local FL model of each user $j$ at each iteration $\mu$,  i.e., $\boldsymbol{\hat w}_{j,\mu}=\boldsymbol{w}_{i^*,\mu}- \boldsymbol{o}$. To enable the FNN to predict each user's local FL model, the FNN must be trained by the online gradient descent method \cite{wu2011convergence}. 

Given the prediction of the users' FL models, the update of the global FL model can be rewritten by
\begin{equation}\label{eq:w2}
\boldsymbol{g}\left(\boldsymbol{a}_{\mu} \right)={\frac{ \sum\limits_{i =1}^{U} K_ia_{i,\mu} \boldsymbol{w}_{i,\mu}+{\sum\limits_{i=1}^U K_i\left(1-a_{i,\mu}\right) \boldsymbol{\hat w}_{i,\mu}}\mathbbm{1}_{\left\{ E_{i,\mu}\le \gamma \right\}}      }{\sum\limits_{i=1}^U K_ia_{i,\mu}+{\sum\limits_{i=1}^U K_i\left(1- a_{i,\mu} \right)}\mathbbm{1}_{\left\{ E_{i,\mu}\le \gamma \right\}}   }},
\end{equation}   
where $\sum\limits_{i =1}^{U} K_ia_{i,\mu} \boldsymbol{w}_{i,\mu}$ is the sum of the local FL models of the users that connect to the BS at iteration $\mu$ and ${\sum\limits_{i=1}^U K_i\left(1-a_{i,\mu}\right) \boldsymbol{\hat w}_{i,\mu}}\mathbbm{1}_{\left\{ E_{i,\mu}\le \gamma \right\}} $ is the sum of the predicted local FL models of the users that are not associated with the BS at iteration $\mu$, $E_{i,\mu}=\frac{1}{2W}\left\|\boldsymbol{\hat w}_{i,\mu}-\boldsymbol{ w}_{i,\mu}\right\|^2$ is the prediction error at iteration $\mu$, and $\gamma$ is the prediction requirement. In (\ref{eq:w2}), when the prediction accuracy of the FNN cannot meet the prediction requirement (i.e., $E_{i,\mu}>\gamma$), the BS will not use the prediction result for updating its global FL model. 
% $\sum\limits_{i \in \mathcal{U}_{1,\mu}} K_i \boldsymbol{w}_{i,\mu}$ is the summarization of the local FL models that are transmitted from the users associated with the BS via RBs while ${\sum\limits_{j \in \mathcal{U}_{0,\mu}} K_j \boldsymbol{\hat w}_{j,\mu}}\mathbbm{1}_{\left\{ E_j\le \gamma \right\}} $ is the summarization of the local FL models of the users that are not allocated any RBs. 
From (\ref{eq:w2}), we can also observe that, using {\color{black}MLP}, the BS can include additional local FL models to generate the global FL model so as to decrease  the FL {\color{black}training loss} and improve FL convergence speed. (\ref{eq:w2}) is used to generate the global FL model in Step c. of the FL training procedure specified in Section II.

The proposed FL algorithm that jointly minimizes the FL convergence time FL {\color{black}training loss} is shown in Algorithm \ref{algorithm}. From Algorithm \ref{algorithm}, we can see that the user selection and RB allocation are optimized at each FL iteration and, hence problem (\ref{eq:max}) is solved at each FL iteration. 
\begin{algorithm}[t]\footnotesize
\caption{Proposed FL over wireless network}   
\label{algorithm}   
\begin{algorithmic}[1] %这个1 表示每一行都显示数字  
\vspace{1pt}  
\ENSURE Local FL model of each user $i$, $\boldsymbol{w}_i$, FNN model for the prediction of local FL model, $\boldsymbol{v}^\textrm{in}, \boldsymbol{v}^\textrm{out}, \boldsymbol{b}_\vartheta, \boldsymbol{b}_o$, user $i^*$ that always connects to the BS. \\ %算法的输出：Output  
\vspace{1pt}  
\FOR {iteration $\mu$}
\vspace{1pt}
\STATE  Each user $i$ trains its local FL model to obtain $\boldsymbol{w}_{i,\mu}$.    
\vspace{1pt}  
\STATE Each user $i$ calculates the change of gradient of the local FL model $\boldsymbol{e}_{i,\mu}$ and sends $\left|\boldsymbol{e}_{i,\mu}\right|$ to the BS.   
\vspace{1pt}  
\STATE The BS calculates $p_{i,\mu}$ using (\ref{eq:probability}) and determines $\boldsymbol{a}_\mu$.  
\vspace{1pt} 
\STATE The BS determines $\boldsymbol{R}_\mu$ in (\ref{eq:RB}). 
\vspace{1pt} 
\STATE The selected users ($a_{i,\mu}=1$) transmit their local FL models to the BS based on $\boldsymbol{R}_\mu$ and $\boldsymbol{a}_\mu$.
\vspace{1pt} 
\STATE The BS uses {\color{black}MLP} to estimate the local FL models of the users who are not associated with the BS ($a_{i,\mu}=0$).
\vspace{1pt} 
\STATE The BS calculates the global FL model $\boldsymbol{g}$ using (\ref{eq:w2}).
  \vspace{1pt} 
\STATE The BS uses collected local FL models to train the {\color{black}MLP}.
\ENDFOR  
\end{algorithmic}
\end{algorithm}  
 
\subsection{Convergence, Implementation, and Complexity Analysis}\label{se:CIC}
\subsubsection{Convergence Analysis}Next, we analyze the convergence of the proposed FL algorithm.  We first assume that $F\left(\boldsymbol{g}_\mu\right)=\frac{1}{K} \sum\limits_{i=1}^{U}  \sum\limits_{k =1 }^{K_i}{f\left( {\boldsymbol{g}_\mu,{\boldsymbol{x}_{ik}},{{y}_{ik}}} \right)}$ and $F_i\left(\boldsymbol{g}_\mu\right)=\frac{1}{K_i}\sum\limits_{k =1 }^{K_i}{f\left( {\boldsymbol{g}_\mu,{\boldsymbol{x}_{ik}},{{y}_{ik}}} \right)}$ where $\boldsymbol{g}_\mu$ is short for $\boldsymbol{g}\left(\boldsymbol{a}_\mu\right)$. We also assume that $\| \nabla F_i\left(\boldsymbol{g}_\mu\right) \|^2 \le\zeta_{\mu}^1+\zeta_{\mu}^2 \| \nabla F \left(\boldsymbol{g}_{\mu} \right)\|^2$ and $\| \nabla \hat F_i\left(\boldsymbol{g}_{\mu} \right)\| = \varsigma_{\mu}^1+\varsigma_{\mu}^2 \| \nabla F\left(\boldsymbol{g}_{\mu}\right)\|$ where $\| \nabla \hat F_i\left(\boldsymbol{g}_{\mu}\right)\|$ is the gradient deviation that is caused by the prediction inaccuracy of user $i$'s local FL model. We also assume that $\boldsymbol{g}^*$ is the optimal global FL model that is achieved by the FL algorithm that can collect all users' local FL models at each iteration.
  The convergence of the proposed FL algorithm at each iteration $\mu$ is given by the following theorem.
 
 \begin{theorem}\label{th:1}
\emph{Given the optimal global FL model $\boldsymbol{g}^*$, the gradient deviation caused by the prediction inaccuracy of users' local FL models, $\| \nabla \hat F_i\left(\boldsymbol{g}_\mu\right) \|^2$, and the learning rate $\lambda=\frac{1}{L}$, the upper bound of $\mathbb E  \left[F\left(\boldsymbol{g}_{\mu+1}\right) -  F\left(\boldsymbol{g}^*\right)\right]$ can be given by
\begin{equation}\label{eq:th1}
\begin{split}
\mathbb E\left[ F\left(\boldsymbol{g}_{\mu+1}\right)-F\left(\boldsymbol{g}^{*}\right)\right]  \leq \varpi^1_\mu+\varpi^2_\mu \mathbb E\left(F \left(\boldsymbol{g}_{\mu} \right)-F\left(\boldsymbol{g}^{*}\right)\right).
\end{split}
\end{equation}
where $\varpi^1_\mu=\frac{\zeta_{\mu}^1\left(9K- 8\mathbb E\left(A\right)\right)}{2LK}$ determines the global FL model at convergence and convergence speed with $\mathbb E\left(A\right)={\sum\limits_{i=1}^U K_i\left(1-p_{i,\mu}\right)\mathbbm{1}_{\left\{ E_{i,\mu}\le \gamma \right\}}+\sum\limits_{i=1}^U K_ip_{i,\mu}}$ and
$\varpi^2_\mu=\left(1-\frac{\vartheta}{L}+\frac{\vartheta\zeta_{\mu}^2\left(9K- 8\mathbb E\left(A\right)\right)}{LK}\right)$ also affects the convergence speed. 
%\begin{equation}\label{eq:theorem1}
%\begin{split}
%\mathbb E [ F(\boldsymbol{g}_{t+1})-F(\boldsymbol{g}^{*})]  \leq
%\underbrace{\frac  {\zeta_1}  {2LK}
% \sum\limits_{i=1}^U K_i q_i\left(\boldsymbol{r}_{i}, P_{i} \right) \frac{1-A^t}{1-A}}_{\textrm{Impact of wireless factors on FL convergence}}+A^t  \mathbb E(F (\boldsymbol{g}_{0} )-F(\boldsymbol{g}^{*})),
% \end{split}
%\end{equation}
%where $A=1-\frac{\mu}{L}+ \frac  {\mu\zeta_2} { {LK}}\sum\limits_{i=1}^U  K_i q_i\left(\boldsymbol{r}_{i}, P_{i}\right)$.
}
\end{theorem}
\begin{proof} See Appendix A.
\end{proof}

From Theorem \ref{th:1}, we can see that as $ \varpi^1_\mu$ decreases, the gap between $\boldsymbol{g}_{\mu}$ and $\boldsymbol{g}^*$ decreases. In particular, as $ \varpi^1_\mu=0$, the proposed FL algorithm will converge to the optimal FL model $\boldsymbol{g}^*$ and achieve the optimal {\color{black}training loss}. Hereinafter, the gap between the global FL model that the proposed FL algorithm converges to, $\boldsymbol{g}_{\mu}$, and the optimal global FL model $\boldsymbol{g}^*$ is referred as the convergence accuracy. $ \varpi^1_\mu$ is used to capture the convergence accuracy. From Theorem \ref{th:1}, we can also see that $ \varpi^1_\mu$ and $\varpi^2_\mu$ jointly determine the convergence speed. This is because as $ \varpi^1_\mu$ and $\varpi^2_\mu$ decrease, the value of $\mathbb E \left[ F\left(\boldsymbol{g}_{\mu+1}\right)-F\left(\boldsymbol{g}^{*}\right)\right] $ will be smaller than the value of $\mathbb E\left[ F\left(\boldsymbol{g}_{\mu}\right)-F\left(\boldsymbol{g}^{*}\right)\right]$. Hence, as $ \varpi^1_\mu$ and $\varpi^2_\mu$ decrease, the speed of $\boldsymbol{g}_{\mu+1}$ converging to $\boldsymbol{g}^{*}$ increases.
 Theorem \ref{th:1} shows that the prediction of local FL models affects both convergence speed and convergence accuracy. In particular, as the prediction errors of local FL models $E_{i,\mu}$ is smaller than $\gamma$, then the predicted local FL models can be used to update the global FL model hence improving the convergence speed and accuracy of the proposed FL algorithm. From Theorem \ref{th:1}, we can also see that the number of training data samples $K_i$ and the connection probability $p_{i,\mu}$ of each user $i$, also affect the convergence speed and accuracy.     
 
Based on Theorem \ref{th:1}, we have the following observations.
\begin{corollary}\label{le:1}
\emph{If the BS cannot predict the local FL models of the users that are not allocated any RBs, the upper bound of $\mathbb E\left[F\left(\boldsymbol{g}_{\mu+1}\right) -  F\left(\boldsymbol{g}^*\right)\right]$ can be given by
\begin{equation}\label{eq:le1}
\begin{split}
&\mathbb E \left[ F\left(\boldsymbol{g}_{\mu+1}\right)-F\left(\boldsymbol{g}^{*}\right)\right]   \leq \frac{\zeta_{\mu}^1\left(9K- 8\sum\limits_{i=1}^U K_ip_{i,\mu}\right)}{2LK}\\
&\;+\left(1-\frac{\vartheta}{L}+\frac{\vartheta\zeta_{\mu}^2\left(9K- 8\sum\limits_{i=1}^U K_ip_{i,\mu}\right)}{LK}\right)\mathbb E\left(F \left(\boldsymbol{g}_{\mu} \right)-F\left(\boldsymbol{g}^{*}\right)\right).
\end{split}
\end{equation}
}
\end{corollary}
\begin{proof}
If the BS cannot predict the users' local FL models, we have $\mathbbm{1}_{\left\{ E_{i,\mu}\le \gamma \right\}}=0$ and hence $\mathbb E\left(A\right)=\sum\limits_{i=1}^U K_ip_{i,\mu}$. Substituting $\mathbb E\left(A\right)=\sum\limits_{i=1}^U K_ip_{i,\mu}$ into (\ref{eq:th1}), we can obtain (\ref{eq:le1}). This completes the proof.
\end{proof}

From Corollary \ref{le:1} and Theorem \ref{th:1}, we can see that $\varpi^1_\mu \leq  \frac{\zeta_{\mu}^1\left(9K- 8\sum\limits_{i=1}^U K_ip_{i,\mu}\right)}{2LK}$ and $\varpi^2_\mu \leq 1-\frac{\vartheta}{L}+\frac{\vartheta\zeta_{\mu}^2\left(9K- 8\sum\limits_{i=1}^U K_ip_{i,\mu}\right)}{LK}$, which implies that the prediction of local FL models can improve the convergence speed and convergence accuracy. Hence, this is the lower bound of the expected convergence of the proposed FL algorithm. Corollary \ref{le:1} also shows that there exists a gap between $\boldsymbol{g}_\mu$ and $\boldsymbol{g}^*$ at convergence. This gap is caused by the probabilistic user association. %However, compared to the standard FL algorithm in \cite{} that all users have equal probabilities for local FL model transmission, the proposed probabilistic user association scheme can improve the convergence speed and convergence accuracy, as shown in the following proposition.   

\begin{corollary}\label{le:2}
\emph{If the prediction accuracy of each user's local FL model satisfies $E_{i,\mu} \le \gamma$, the upper bound of $\mathbb E \left[F\left(\boldsymbol{g}_{\mu+1}\right) -  F\left(\boldsymbol{g}^*\right)\right]$ can be given by
\begin{equation}\label{eq:le2}
\begin{split}
\mathbb E\left[ F\left(\boldsymbol{g}_{\mu+1}\right)-F\left(\boldsymbol{g}^{*}\right)\right]  & \leq \frac{\zeta_{\mu}^1}{2L}\\
&\!\!\!\!\!\!\!\!\!\!\!\!\!\!\!\!\!\!\!\!\!+\left(1-\frac{\vartheta}{L}+\frac{\vartheta\zeta_{\mu}^2}{L}\right)\mathbb E\left(F \left(\boldsymbol{g}_{\mu} \right)-F\left(\boldsymbol{g}^{*}\right)\right).
\end{split}
\end{equation}
}
\end{corollary}
\begin{proof}
If $E_{i,\mu} \le \gamma$, we have $\mathbbm{1}_{\left\{ E_{i,\mu}\le \gamma \right\}}=1$ and hence $\mathbb E\left(A\right)=K$. Substituting $\mathbb E\left(A\right)=K$ into (\ref{eq:th1}), (\ref{eq:le2}) can be obtained. This completes the proof.
\end{proof}

From Corollary \ref{le:2}, we can see that if $E_{i,\mu} \le \gamma$, the convergence accuracy only depends on $\zeta_{\mu}^1$ and $\zeta_{\mu}^2$. Hence, this is an upper bound of the expected convergence of the proposed FL algorithm.  

Theorem \ref{th:1} derives the convergence accuracy and rate for the proposed FL algorithm that uses full gradient descent. Next, we derive the convergence accuracy and rate of the proposed FL algorithm when it uses a \emph{stochastic gradient descent (SGD)} \cite{zinkevich2010parallelized} method to update the local FL models. We assume that each user will select a subset of $M$ In particular, using the SGD update method, the update of local FL model at each user in (\ref{eq:g}) can be rewritten by
\begin{equation}\label{eq:gsgd}
{\color{black}
\boldsymbol{w}_{i,\mu+1}=\boldsymbol{ g}\left(\boldsymbol{a}_{\mu}\right) -{\lambda} \sum\limits_{k \in \mathcal{K}_{i,\mu}} \nabla {f\left( { \boldsymbol{ g}\left(\boldsymbol{a}_\mu\right) ,{\boldsymbol{x}_{ik}},{\boldsymbol{y}_{ik}}} \right)},}
\end{equation}  
where $\mathcal{K}_{i,\mu}$ is the subset of $M$ training data samples selected from user $i$'s training dataset $\mathcal{K}_i$. 
%We also assume that the probability that each user $i$ selects each sample $k$ to update the local FL models is $q_{ik}$. 
Hereinafter, we use $\nabla {f_{ik}\left( { \boldsymbol{ g}_\mu} \right)}$ for $\nabla {f\left( { \boldsymbol{ g}\left(\boldsymbol{a}_\mu\right) ,{\boldsymbol{x}_{ik}},{\boldsymbol{y}_{ik}}} \right)}$. {\color{black} The users and the BS will exchange FL model parameters once the users implement one SGD update in (\ref{eq:gsgd}).}
%We also assume that $\sum\limits_{k=1}^{K_i}q_{ik} \left\|\nabla {f_{ik}\left( { \boldsymbol{ g}_\mu} \right)}\right\|^2= \nu_{i,\mu}^1+\nu_{i,\mu}^2 \| F \left(\boldsymbol{g}_{\mu} \right)\|^2$ where $q_{ik}$ is the probability that training data sample $k$ is selected for training user $i$'s local FL model. 
Given (\ref{eq:gsgd}), the convergence accuracy and rate of the proposed FL scheme when it uses the SGD update method is given in the following theorem.

% For the proposed FL algorithm that uses  the convergence accuracy and rate is given in the following theorem.  

 \begin{theorem}\label{th:2}
\emph{Given the optimal global FL model $\boldsymbol{g}^*$, the learning rate $\lambda=\frac{1}{L}$, and the subset of training data samples used to update each user $i$'s local FL model at iteration $\mu$, $\mathcal{K}_{i,\mu}$, the upper bound of $\mathbb E \left[F\left(\boldsymbol{g}_{\mu+1}\right) -  F\left(\boldsymbol{g}^*\right)\right]$ for the proposed FL algorithm with SGD can be given by
\begin{equation}\label{eq:th2}
\begin{split}
\mathbb E \left[ F\left(\boldsymbol{g}_{\mu+1}\right)-F\left(\boldsymbol{g}^{*}\right)\right]  \leq \psi ^{\textrm{1}}_\mu+\psi^{\textrm{2}}_\mu \mathbb E\left(F \left(\boldsymbol{g}_{\mu} \right)-F\left(\boldsymbol{g}^{*}\right)\right).
\end{split}
\end{equation}
where $\psi^1_\mu=\frac{\zeta_{\mu}^1\left(9K- 8\mathbb E\left(A^\textrm{S}\right)\right)}{2LK}$ and $\psi^2_\mu=\left(1-\frac{\vartheta}{L}+\frac{\vartheta\zeta_{\mu}^2\left(9K- 8\mathbb E\left(A^\textrm{S}\right)\right)}{LK}\right)$ with $\mathbb E\left(A^\textrm{S}\right)=\sum\limits_{i=1}^U K_i\left(1-p_{i,\mu}\right)\mathbbm{1}_{\left\{ E_{i,\mu}\le \gamma \right\}}+\sum\limits_{i=1,i \ne i^*}^U\sum\limits_{k \in \mathcal{K}_{i,\mu}}  p_{i,\mu}+K_{i^*}$
%\begin{equation}\label{eq:theorem1}
%\begin{split}
%\mathbb E [ F(\boldsymbol{g}_{t+1})-F(\boldsymbol{g}^{*})]  \leq
%\underbrace{\frac  {\zeta_1}  {2LK}
% \sum\limits_{i=1}^U K_i q_i\left(\boldsymbol{r}_{i}, P_{i} \right) \frac{1-A^t}{1-A}}_{\textrm{Impact of wireless factors on FL convergence}}+A^t  \mathbb E(F (\boldsymbol{g}_{0} )-F(\boldsymbol{g}^{*})),
% \end{split}
%\end{equation}
%where $A=1-\frac{\mu}{L}+ \frac  {\mu\zeta_2} { {LK}}\sum\limits_{i=1}^U  K_i q_i\left(\boldsymbol{r}_{i}, P_{i}\right)$.
}
\end{theorem}
\begin{proof} See Appendix B.
\end{proof}
From Theorem \ref{th:2}, we can see that, in order to guarantee the convergence of the proposed FL with SGD, $\psi^2_\mu$ must be smaller than 1 (i.e., $\psi^2_\mu<1$). Since $\psi^2_\mu$ depends on the predicted local FL models, the local FL model of user $i^*$, user selection probability $p_{i,\mu}$, and the size of training dataset $\mathcal{K}_{i,\mu}$, we can adaptively adjust these parameters so as to guarantee the convergence of the proposed FL with SGD. For instance, we can adjust $p_{i,\mu}$ so as to maximize $\mathbb E\left(A^\textrm{S}\right)$ hence decreasing $\psi^2_\mu$.

\subsubsection{Implementation Analysis}
With regards to the implementation of the proposed algorithm, the BS must: a) Determine the user selection policy, b) Use an optimization algorithm to find the optimal RB allocation for each user, and c) Use {\color{black}MLP} to predict the users' local FL models. To determine the user selection policy, the BS requires $\|\boldsymbol{e}_{i,\mu}\|$ of each user $i$ at each iteration $\mu$. Hence, each user must transmit $\|\boldsymbol{e}_{i,\mu}\|$ to the BS at each iteration. Since $\|\boldsymbol{e}_{i,\mu}\|$ is a scalar, the data size of which is much smaller than the data size of the local FL models that the users must transmit to the BS during each iteration. In consequence, we can ignore the overhead of each user transmitting $\|\boldsymbol{e}_{i,\mu}\|$ to the BS. To use an optimization algorithm for optimizing RB allocation, the BS needs to calculate the total transmission delay $l_{i}^\textrm{U}\left(\boldsymbol{r}_{i,\mu}\right) +l_{i}^\textrm{D} $ of each user $i$ over each RB. The BS can use channel estimation methods to learn the signal-to-interference-plus-noise ratio over each RB so as to calculate $l_{i}^\textrm{U}\left(\boldsymbol{r}_{i,\mu}\right) +l_{i}^\textrm{D} $. To train the {\color{black}MLP} that is used for the predictions of users' local FL models, the BS will use the local FL models that are transmitted from the users who have RBs. These local FL models are originally used for the update of the global FL model and, hence, the BS does not require any additional information for training {\color{black}MLP}. {\color{black}Although our analytical results and the proposed solutions are considered for a network that consists of only one BS, one can easily extend them to a network that consists of multiple BSs by considering user association with different BSs.} 

\subsubsection{Complexity Analysis}
With regards to the complexity of the proposed algorithm, we first analyze the complexity of the interior-point method that is used to find the optimal RB allocation for each user. Let $L_\textrm{O}$ be the number of iterations until the interior-point method converges. The complexity of the interior-point method is $\mathcal O\left(L_\textrm{O}UR\right)$ \cite{boyd2004convex}. Therefore, the complexity of the interior-point method, which depends on both the number of RBs and users, is linear. The complexity of training {\color{black}MLP} depends on training data samples and number of users. Since {\color{black}MLP} is trained by the BS which has enough computational resource, the overhead of training {\color{black}MLP} can be ignored.        
\section{Simulation Results and Analysis}
For our simulations, we consider a circular network area having a radius $r=500$ m with one BS at its center servicing $U = 15$ uniformly distributed users. {\color{black}The value of  the inter-cell interference at each RB is randomly selected from $\left[10^{-4}, 0.01\right]$}. The FL algorithm is used for two learning tasks: a) the function fitting task where an FL algorithm is used for {\color{black}regression} and b) the {\color{black}classification} task where an FL algorithm is used to identify the handwritten digits from 0 to 9. For the {\color{black}classification} task, 
each user trains a MLP using the MNIST dataset \cite{MNIST}. The size of neuron weight matrices are $784 \times 50$ and $50 \times 10$. The BS also implements a MLP  for each user to predict its local FL model parameters. The {\color{black}MLP} is generated based on the MATLAB machine learning toolbox \cite{MNISTmatlab}. 1,000 handwritten digits are used to test the trained FL algorithms. The other parameters used in simulations are listed in Table~I.
%Low-density parity-check code (LDPC) \cite{ryan2004introduction} coded quadrature phase shift keying (QPSK) modulation is used for data transmission. The bit error rate $\mathbbm{P}_i\left( \boldsymbol{r}_{i},P_i\right)$ in (\ref{eq:nik}) is approximated by a piecewise linear approximation.
  For comparison purposes, we use two baselines: a) an FL algorithm that uses the proposed user association policy without the prediction of users' local FL models at each given learning step and b) a standard FL algorithm in \cite{konevcny2016federated} that randomly determines user selection and resource allocation without using {\color{black}MLP} to estimate the local FL model parameters of each user at each given learning step. Hereinafter, \emph{the proposed FL refers to the proposed FL algorithm that uses stochastic gradient descent method to train the local FL models}. The simulation results are averaged over a large number of independent runs.

\begin{table}\footnotesize
  \newcommand{\tabincell}[2]{\begin{tabular}{@{}#1@{}}#2.2\end{tabular}}
\renewcommand\arraystretch{1}
 \caption{
    \vspace*{-0.05em}SYSTEM PARAMETERS}\vspace*{-0.6em}
\centering  
\begin{tabular}{|c|c|c|c|}% ±íÊŸž÷ÁÐÔªËØ¶ÔÆë·œÊœ£¬left-l,right-r,center-c
\hline
\textbf{Parameter} & \textbf{Value} & \textbf{Parameter} & \textbf{Value} \\
\hline
$\alpha$&2  &$N_0$& -174 dBm/Hz \\
\hline
 $P$ & 1 W & $B$& 1 MHz  \\
\hline
 $R$& 5 &$B^\textrm{D}$& 20 MHz \\
\hline
$ N $ &5& $P_\textrm{B}$ & 1 W \\
\hline
$N_\textrm{out}$&$10$& $K_i$ & 500  \\
\hline
$N_\textrm{in} $&$784 $& $\gamma$ &0.01\\
\hline
$W $&$5000 $& $\gamma_\textrm{R}$ &5\\
\hline
%$\omega_i$ & 40&$\gamma_\textrm{E}$ & 0.02 J  \\
%\hline
 %$B$& 20 MHz &&    \\ 
%\hline
\end{tabular}
 \vspace{-0.5cm}
\end{table}

\begin{figure*}[!t]
\centering
\vspace{0cm}
\subfigure[]{
\label{fig3a} %% label for first subfigure
\includegraphics[width=5.8cm]{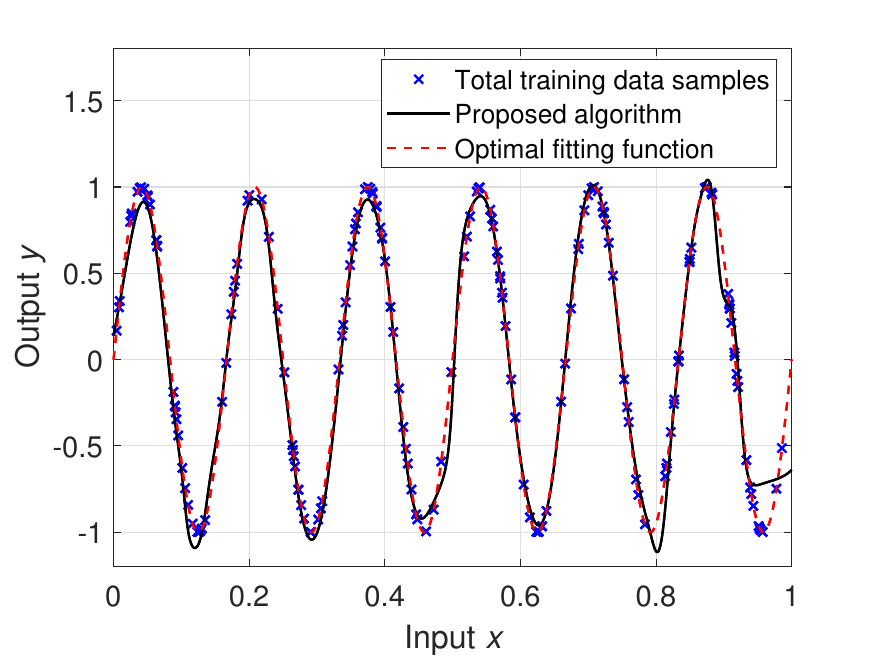}}
\subfigure[]{ 
\label{fig3b} %% label for second subfigure
\includegraphics[width=5.8cm]{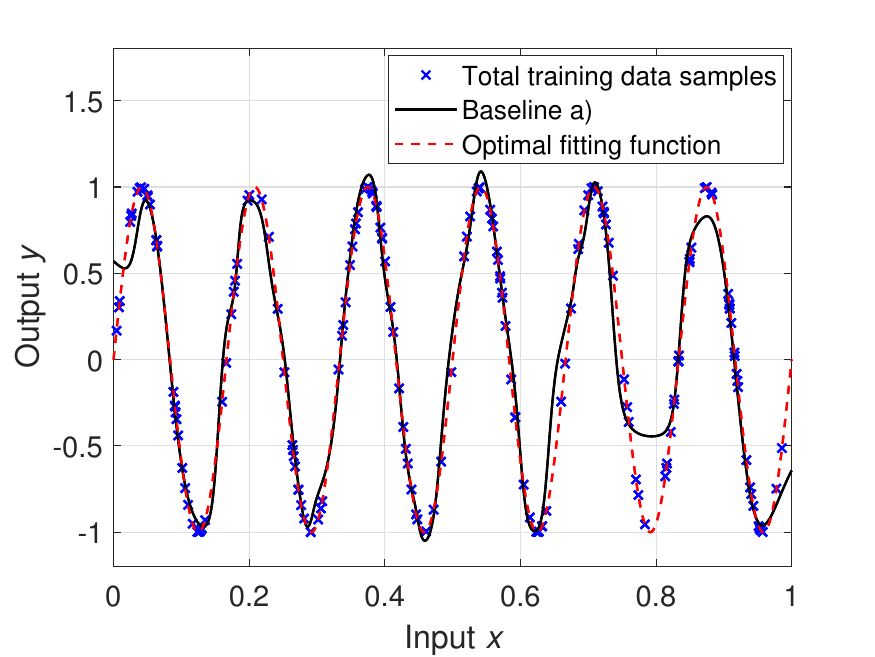}}
\subfigure[]{ 
\label{fig3b} %% label for second subfigure
\includegraphics[width=5.8cm]{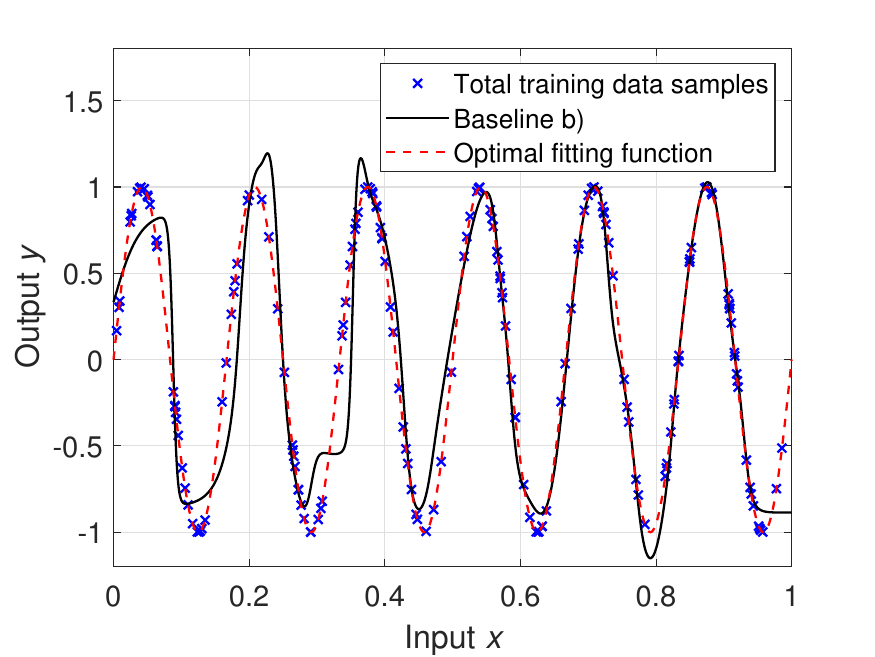}}
  \vspace{-0.1cm}
%\caption{\footnotesize{Downlink and uplink sum-rates vs. the number of users.}}
 \caption{\label{fig8}  An example of implementing FL for {\color{black}regression}.}
 \vspace{-0.3cm}
\end{figure*}

In Fig. \ref{fig8}, we show an example of implementing FL for {\color{black}regression}. In this example, FL is used to approximate the function $y=\sin\left(2\pi x\right)$ where $x\in \left[0,1\right]$ is the FL input and $y$ is the FL output. 15 users jointly implemented the considered FL algorithms and each user has $12$ training data samples. To perform this learning task, each user implements a function fitting neural network \cite{fnetmatlab} based FL algorithm. In Fig. \ref{fig8}, the total training data samples are the training data samples of all users ($15\times12=180$). The optimal fitting function is the target function that the FL algorithm is approximating. From Fig. \ref{fig8}, we can see that the proposed FL algorithm approximates the optimal fitting function better than baselines a) and b). This is because the proposed FL algorithm uses a probabilistic user selection scheme to select the users that transmit their local FL models to the BS at each FL iteration which improves the learning speed. Meanwhile, at each given learning step, the proposed FL algorithm uses {\color{black}MLP} to estimate the local FL models of the users that are not allocated any RBs for their local model parameter transmission so as to include more local FL models to generate the global FL model.

\begin{figure}[!t]
  \begin{center}
   \vspace{0cm}
    \includegraphics[width=8cm]{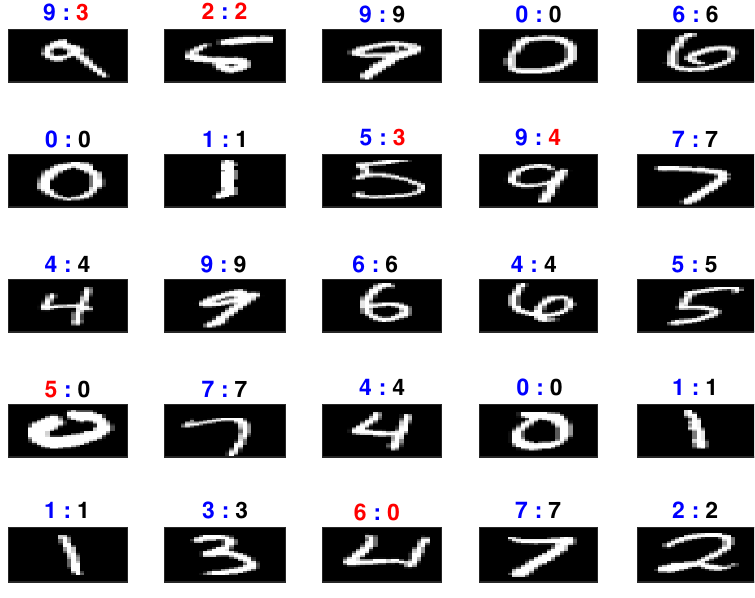}
    \vspace{-0.25cm}
 {\color{black}   \caption{\label{fig1} An example of implementing FL algorithms for {\color{black}classification} of handwritten digits. In this figure, the black digits that are displayed above the handwritten digits are the {\color{black}classification} results of the proposed FL algorithm. The black digits that are displayed below the handwritten digits are the {\color{black}classification} results of the standard FL algorithm. Meanwhile, the red digits in this figure represents the wrong {\color{black}classification} results from all considered FL algorithms.}}
  \end{center}\vspace{-0.6cm}
\end{figure}

Fig. \ref{fig1} shows an example of implementing the considered FL algorithms for the {\color{black}classification} of handwritten digits from 0 to 9. 
 From Fig. \ref{fig1}, we can see that, for a total of {\color{black}25} handwritten digits, the proposed FL algorithm can correctly identify {\color{black}22} digits while the standard FL algorithm can identify {\color{black}20} digits. This is because, at each given learning step, the proposed FL algorithm can use {\color{black}MLP} to estimate users' local FL model parameters. Hence, even though the number of RBs is limited, by employing our proposed FL algorithm, the BS can exploit all users' local FL models to generate the global FL model and, thus improving FL convergence time and decreasing {\color{black}FL training loss}.

% \begin{figure}[!t]
%  \begin{center}
%   \vspace{0cm}
%    \includegraphics[width=10cm]{fig1a.eps}
%    \vspace{-0.25cm}
%    \caption{\label{fig2} FL convergence time changes as the number of users varies.}
%  \end{center}\vspace{-0.9cm}
%\end{figure}

\begin{figure}
\centering
\vspace{0cm}
\subfigure[]{
\label{fig2a} %% label for first subfigure
\includegraphics[width=8cm]{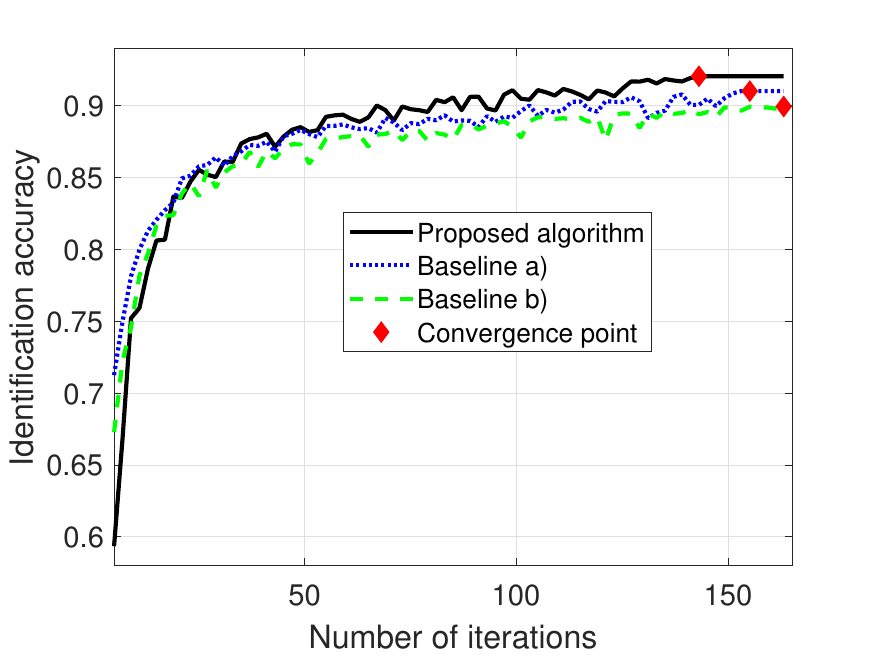}}
\subfigure[]{ 
\label{fig2b} %% label for second subfigure
\includegraphics[width=8cm]{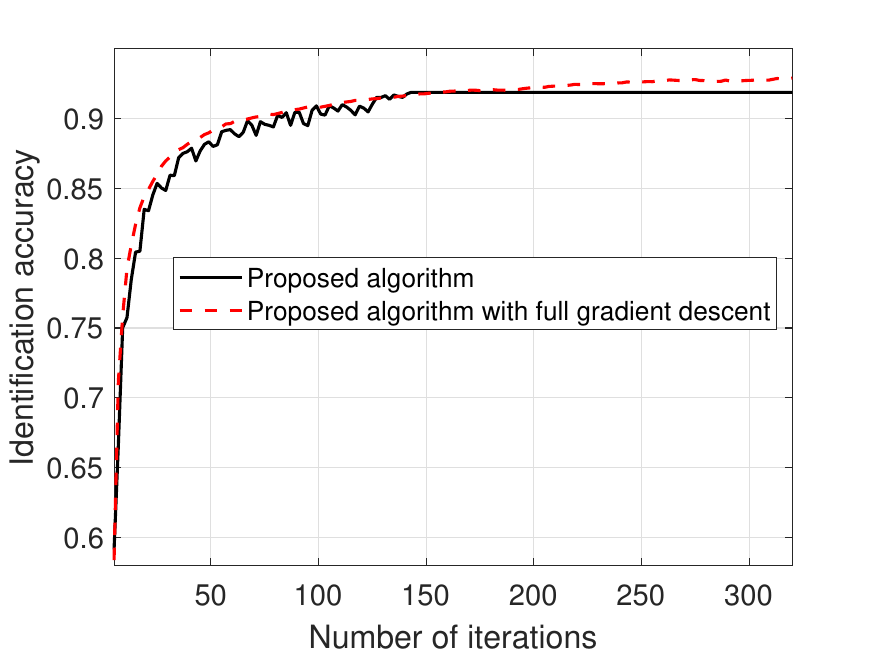}}
  \vspace{-0.2cm}
%\caption{\footnotesize{Downlink and uplink sum-rates vs. the number of users.}}
 \caption{\label{fig2} Identification accuracy as the number of iteration varies.}
 \vspace{-0.6cm}
\end{figure}

In Fig. \ref{fig2}, we show how the FL {\color{black}identification accuracy} changes as time elapses. From this figure, we can see that, as time elapses, the FL {\color{black}identification accuracy} of all considered algorithms increases. This is because the local FL models and the global FL model are trained by the users and the BS  as time elapses. From Fig. \ref{fig2a}, we can see that the proposed FL algorithm can reduce the number of iterations needed for convergence, by, respectively, up to 9\% and 14\% compared to baselines a) and b). The 9\% gain stems from the fact that the proposed FL algorithm uses {\color{black}MLP} to estimate the local FL model parameters of the users that are not allocated any RBs for local FL model transmission at each given learning step. The 14\% gain stems from the fact that the proposed FL algorithm uses the proposed probabilistic user selection scheme to select the users for local FL model transmission and uses the ANNs to estimate the local FL model parameters of the users that do not RBs for local FL model transmission at each given learning step. From Fig. \ref{fig2b}, we can see that the proposed algorithm converges faster than the proposed algorithm with full gradient descent. However, the proposed algorithm with full gradient descent can achieve better {classification accuracy} compared to the proposed algorithm. This is due to the fact that a full gradient descent method uses all training data samples to train the local FL models at each iteration while the SGD method uses a subset of training data samples to train the local FL models.

  \begin{figure}[!t]
  \begin{center}
   \vspace{0cm}
    \includegraphics[width=8cm]{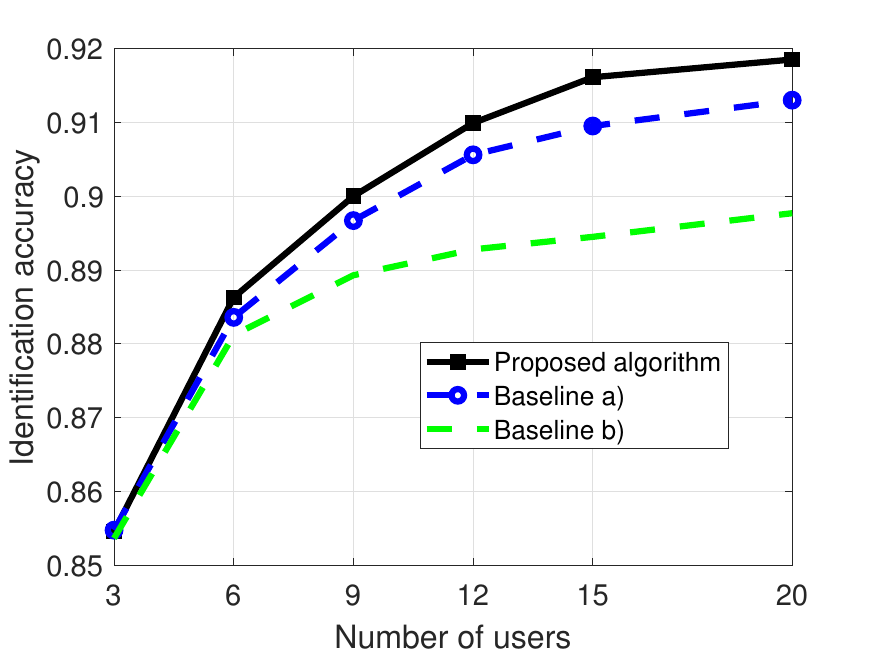}
    \vspace{-0.25cm}
 {\color{black}   \caption{\label{fig4} Training loss changes as the number of users varies.}}
  \end{center}\vspace{-0.6cm}
\end{figure}

{\color{black}Fig. \ref{fig4} shows how the identification accuracy changes as the number of users varies. In this figure, we can see that, as the number of users increases, the FL identification accuracy of all considered algorithms increases. This is because as the number of users increases, the number of data samples used for training FL increases. From Fig. \ref{fig4}, we can also see that, for a network with 20 users, the proposed FL algorithm can improve the {\color{black}identification accuracy} by up to 1\% and 3\%, respectively, compared to baselines a) and b). These gains stem from the fact that, in the proposed FL algorithm, a probabilistic user selection scheme is developed for user selection and local FL model transmission. Meanwhile, to include additional local FL models to generate the global FL model, at each given learning step, the proposed FL algorithm uses ANNs to estimate the local FL model parameters of the users that are not allocated any RBs for local FL model transmission hence improving the {\color{black}classification} accuracy. Fig. \ref{fig4} also shows that, as the number of users increases, the gap between the {\color{black}identification accuracy} resulting from the proposed FL algorithm and the baselines increases. This is because, for the considered baselines, as the number of users increases, the number of users that can transmit their local FL models to the BS remains the same due to the limited number of RBs. In contrast, the proposed FL algorithm can use {\color{black}MLP} to estimate the local FL models of the users that are not allocated any RBs for transmitting their local model parameters at each given learning step and, hence, include more local FL models to generate the global FL model.}

  \begin{figure}[!t]
  \begin{center}
   \vspace{0cm}
    \includegraphics[width=8cm]{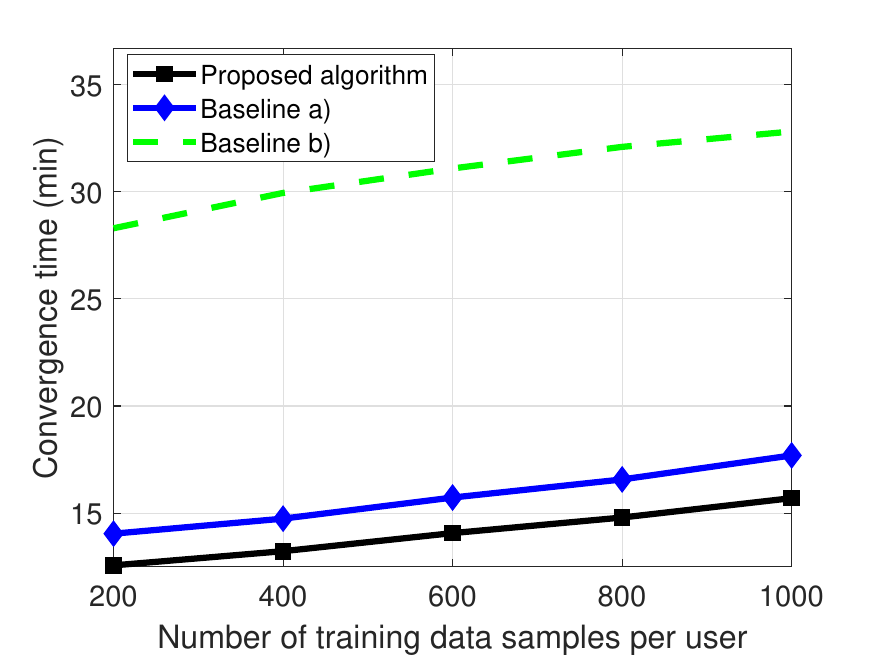}
    \vspace{-0.25cm}
    \caption{\label{fig5} Convergence time changes as the number of training data samples varies.}
  \end{center}\vspace{-0.6cm}
\end{figure}

In Fig. \ref{fig5}, we show how how the convergence time changes as the number of training data samples varies. Fig. \ref{fig5} shows that, as the number of training data samples increases, the convergence time of all considered FL algorithms increases. This is due to the fact that all of the considered FL algorithms use a stochastic gradient descent method to train their local FL models. Hence, as the number of training data samples increases, the various FL algorithms need to use more time to sample the local datasets. Fig. \ref{fig5} also shows that, for a scenario in which each user has 200 training data samples, the proposed FL algorithm can reduce convergence time by up to 11\% and 56\%, compared to baselines a) and b). These gains stem from the fact the proposed FL algorithm uses ANNs to estimate the local FL models of the users that are not allocated any RBs for transmitting their local model parameters at each given learning step, and uses probabilistic user selection scheme to determine the users that will transmit the local FL model parameters to the BS. In addition, the proposed algorithm optimizes RB allocation at each iteration.

  \begin{figure}[!t]
  \begin{center}
   \vspace{0cm}
    \includegraphics[width=8cm]{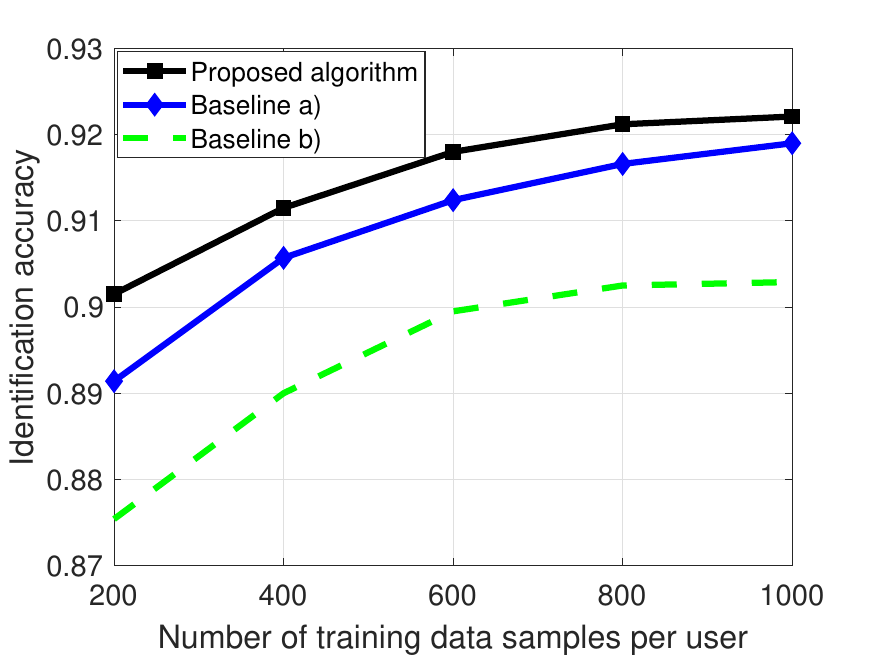}
    \vspace{-0.25cm}
    \caption{\label{fig7} {\color{black}Training loss} changes as the number of training data samples varies.}
  \end{center}\vspace{-0.6cm}
\end{figure}
Fig. \ref{fig7} shows how the handwritten digit identification accuracy changes as the number of training data samples varies. From Fig. \ref{fig7}, we can see that, as the number of training data samples increases from 200 to 800, the {\color{black}identification accuracy} of all considered algorithms increase. This is because, as the number of training data samples increases, each user can use more data samples to train the local FL model thus improving the {\color{black}identification accuracy} of the FL algorithms. As the number of training data samples continues to increase, the {\color{black}identification accuracy} of all considered algorithms improve slowly. This is because 800 training data samples may include all features of the MNIST dataset.   Fig. \ref{fig7} also shows that, as the number of training data samples increases, the gap between the {\color{black}identification accuracy} resulting from the proposed FL algorithm and baseline a) decreases. This is due to the fact that, as the number of training data samples per user increases, each local FL model is trained by a dataset that contains all features of MNIST dataset, and, hence, the BS can use fewer local FL models to generate the global FL model. 

 \begin{figure}[!t]
  \begin{center}
   \vspace{0cm}
    \includegraphics[width=8cm]{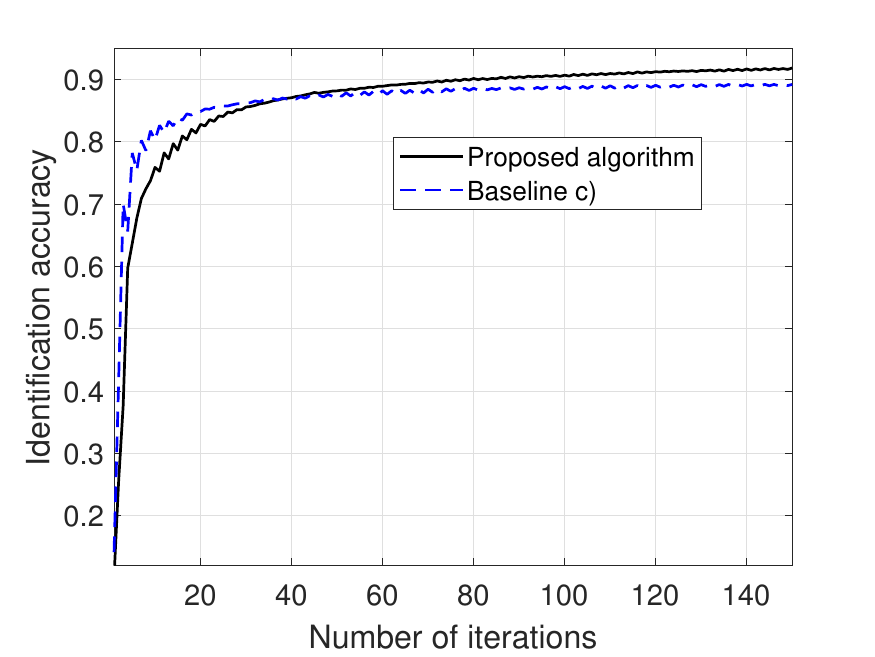}
    \vspace{-0.25cm}
  {\color{black}  \caption{\label{LGD2}Convergence of FL algorithms.}}
  \end{center}\vspace{-0.6cm}
\end{figure}

{\color{black} In Fig. \ref{LGD2}, we show the convergence of both the proposed FL algorithm and the baseline c). In baseline c), the users that transmit local FL models to the BS are selected based on our proposed user selection scheme. For unselected users, the BS will use previous received local FL models for global FL model updates as done in LAG algorithm of \cite{NIPS2018_7752}. From this figure, we can see that, the proposed FL algorithm can improve the identification accuracy by about 1.7\% compared to  baseline c). This is because the proposed FL algorithm estimates the local FL models of the users that are not allocated any RBs.  }

\section{Conclusion}
In this paper, we have developed a novel framework that enables the implementation of FL over wireless networks. {\color{black}The proposed FL framework can be used for various real-world applications such as mobile keyboard prediction, network traffic analysis and prediction, device monitoring, and user behavior analysis for virtual reality users.} We have formulated an optimization problem that jointly considers user selection and resource allocation for the joint minimization of the FL convergence time and the FL {\color{black}training loss}. To solve this problem, we have proposed a probablistic user selection scheme that allows the users whose local FL models have large impacts on the global FL model to associate with the BS with high probability. Given the user selection policy, the uplink RB allocation is determined. To further improve the FL convergence speed, we have studied the use of {\color{black}MLP} to estimate the local FL models of the users that are not allocated any RBs at each given learning step.
Simulation results have shown that the proposed FL algorithm yields significant improvements in terms of convergence time compared to the standard FL algorithm.

  \section*{Appendix}
\subsection{Proof of Theorem \ref{th:1}}\label{Ap:a}
%To prove Theorem \ref{th:1}, we first assume that $F\left(\boldsymbol{g}_\mu\right)=\frac{1}{K} \sum\limits_{i=1}^{U}  \sum\limits_{k =1 }^{K_i}{f\left( {\boldsymbol{g}_\mu,{\boldsymbol{x}_{ik}},{{y}_{ik}}} \right)}$ and $F_i\left(\boldsymbol{g}_\mu\right)=\frac{1}{K_i}\sum\limits_{k =1 }^{K_i}{f\left( {\boldsymbol{g}_\mu,{\boldsymbol{x}_{ik}},{{y}_{ik}}} \right)}$ where $\boldsymbol{g}_\mu$ is short for
%$\boldsymbol{g}\left(\boldsymbol{a}_\mu\right)$. 
Since the loss function ${f\left( {\boldsymbol{g}_\mu,{\boldsymbol{x}_{ik}},{{y}_{ik}}} \right)}$ is strongly convex and twice-continuously differentiable, we have the following observations:
\begin{itemize}
\item $\nabla F\left( \boldsymbol{g}_\mu\right)$ is uniformly Lipschitz continuous with respect to $\boldsymbol{g}_\mu$ \cite{friedlander2012hybrid} and, hence, we can find a positive constant $L$, such that 
\begin{equation}\label{itproofas1}
\|\nabla  F\left(\boldsymbol{g}_{\mu+1} \right)  - \nabla F\left( \boldsymbol{g}_{\mu} \right)\|
\leq  L\| \boldsymbol{g}_{\mu+1} - \boldsymbol{g}_\mu\|,
\end{equation}
where $\| \boldsymbol{g}_{\mu+1} - \boldsymbol{g}_{\mu}\|$ is the norm of $ \boldsymbol{g}_{\mu+1} - \boldsymbol{g}_{\mu}$.    

\item Since $F\left( \boldsymbol{g}\right)$ is strongly convex, we have
\begin{equation}\label{itproofas2}
F\left(\boldsymbol{g}_{\mu+1} \right) \geq F\left(\boldsymbol{g}_{\mu} \right)
+\left( \boldsymbol{g}_{\mu+1} - \boldsymbol{g}_{\mu}\right)^{T}  \nabla F\left(\boldsymbol{g}_{\mu}\right)
+\frac  {\vartheta  } 2 \| \boldsymbol{g}_{\mu+1} - \boldsymbol{g}_{\mu}\|^2.
\end{equation}
where $\vartheta$ is a positive parameter. By minimizing both sides of \eqref{itproofas2} with respect to 
$\boldsymbol{g}_{\mu+1}$, we have 
\begin{equation}\label{eq:Fg}
  F\left(\boldsymbol{g}^* \right) \geq 
F\left (\boldsymbol{g}_{\mu} \right)-\frac{1}{2\vartheta}\|\nabla F\left(\boldsymbol{g}_{\mu}\right)\|^2.
\end{equation}

\item Since $F\left( \boldsymbol{g}\right)$ is twice-continuously differentiable, we have
\begin{equation}\label{itproofas2_2}
 \vartheta \boldsymbol{I} \preceq
\nabla^2 F \left(\boldsymbol{g} \right) \preceq L \boldsymbol{I}.
\end{equation}

\item We also assume that $\| \nabla F_i\left(\boldsymbol{g}_\mu\right) \|^2\leq \zeta_{\mu}^1+\zeta_{\mu}^2  \| \nabla F \left(\boldsymbol{g}_{\mu} \right)\|^2$ with $\zeta_{i,\mu}^1\ge 0$ and $\zeta_{i,\mu}^2\ge 1$.

\end{itemize} 

Using the second-order Taylor expansion, $F \left(\boldsymbol{g}_{\mu+1} \right)$ can be rewritten as follows:
\begin{equation}\label{itproofas3}
\begin{split}
F \left(\boldsymbol{g}_{\mu+1} \right) 
=&F\left(\boldsymbol{g}_{\mu} \right)
+\left( \boldsymbol{g}_{\mu+1} - \boldsymbol{g}_{\mu}\right)^{T}  \nabla F\left(\boldsymbol{g}_{\mu} \right)\\
&+\frac  {1} 2\left( \boldsymbol{g}_{\mu+1} - \boldsymbol{g}_{\mu}\right)^T
\nabla^2 F\left(\boldsymbol{g} \right)
\left( \boldsymbol{g}_{\mu+1} - \boldsymbol{g}_{\mu}\right),\\
&\!\!\!\!\!\!\!\!\!\!\!\mathop \leq \limits^{\left( a \right)} F\left(\boldsymbol{g}_{\mu}\right)
+\left( \boldsymbol{g}_{\mu+1} - \boldsymbol{g}_{\mu}\right)^{T}  \nabla F\left(\boldsymbol{g}_{\mu} \right)
+\frac  {L } 2 \| \boldsymbol{g}_{\mu+1} - \boldsymbol{g}_{t}\|^2,
\end{split}
\end{equation}
where (a) stems from the fact that $\nabla^2 F \left(\boldsymbol{g} \right) \preceq L \boldsymbol{I}$. Based on (\ref{eq:w}) and (\ref{eq:g}), we have $\boldsymbol{g}_{\mu+1} - \boldsymbol{g}_{\mu}=\lambda\left(  \nabla F\left(\boldsymbol{g}_{\mu} \right)-\boldsymbol{e}_\mu\right)$ where $\boldsymbol{e}_\mu$ is a gradient deviation caused by the users that do not transmit their local FL models to the BS at iteration $\mu$ and the prediction errors of the local FL models that are estimated by the BS at iteration $\mu$. In particular, $\boldsymbol{e}_\mu$ can be expressed as
\begin{equation}\label{eq:errormu}
\begin{split}
\boldsymbol{e}_\mu=&\nabla F\left(\boldsymbol{g}_{\mu} \right)-\frac{\sum\limits_{i=1}^U K_ia_{i,\mu}\nabla F_i\left(\boldsymbol{g}_{\mu} \right) }{\sum\limits_{i=1}^U K_i\left(1-a_{i,\mu}\right)\mathbbm{1}_{\left\{ E_{i,\mu}\le \gamma \right\}}+\sum\limits_{i=1}^U K_ia_{i,\mu}}\\
&-\frac{ \sum\limits_{i=1}^U K_i\left(1-a_{i,\mu}\right)\left( \nabla F_i\left(\boldsymbol{g}_{\mu} \right)+ \nabla \hat F_i\left( \boldsymbol{g}_{\mu}\right) \right)\mathbbm{1}_{\left\{ E_{i,\mu}\le \gamma \right\}}}{\sum\limits_{i=1}^U K_i\left(1-a_{i,\mu}\right)\mathbbm{1}_{\left\{ E_{i,\mu}\le \gamma \right\}}+\sum\limits_{i=1}^U K_ia_{i,\mu}},
%\le&\nabla F (\boldsymbol{g}_{\mu} )-\frac{ \sum\limits_{i=1}^U K_i\left(1-a_{i,\mu}\right)\left( \nabla F_i(\boldsymbol{g}_{\mu} )+ \nabla \hat F_i\left( \boldsymbol{g}_{\mu}\right) \right)\mathbbm{1}_{\left\{ E_{j,\mu}\le \gamma \right\}} +\sum\limits_{i=1}^U K_ia_{i,\mu}\nabla F_i(\boldsymbol{g}_{\mu} ) }{K},
\end{split}
\end{equation}
where $\frac{1}{K}\sum\limits_{i=1}^U K_i\left(1-a_{i,\mu}\right)\left( \nabla F_i\left(\boldsymbol{g}_{\mu}\right)+ \nabla \hat F_i\left( \boldsymbol{g}_{\mu}\right) \right)\mathbbm{1}_{\left\{ E_{i,\mu}\le \gamma \right\}} $ is the sum of the gradients in the local FL models that are estimated by the BS while $\frac{1}{K}\sum\limits_{i=1}^U K_ia_{i,\mu}\nabla F_i\left(\boldsymbol{g}_{\mu}\right) $ is the sum of the gradients in the local FL models that are transmitted from the users. 
% For example, if the BS can accurately predict the  users transmit their local FL models to the BS, i.e., $a_{i,\mu}=1,  \forall i \in \mathcal{U}$, we have $\boldsymbol{e}_\mu=\nabla F (\boldsymbol{g}_{\mu} )-\frac{\sum\limits_{i = 1}^U K_i\nabla F_i(\boldsymbol{g}_{\mu} )   }{{\sum\limits_{i = 1}^U K_i}} =0$.  

Let the learning rate $\lambda=\frac{1}{L}$. Based on (\ref{itproofas3}), $\mathbb E\left[ F\left(\boldsymbol{g}_{\mu+1}\right)\right]$ can be expressed as
 \begin{equation}\label{itproofas3_2}
\begin{split}
\mathbb E\left[ F\left(\boldsymbol{g}_{\mu+1}\right)\right]
\leq 
&\mathbb E\left(F\left(\boldsymbol{g}_{\mu} \right)\right)
-\frac{1}{L}\left(\|\nabla F \left(\boldsymbol{g}_{\mu} \right)\|^2-\boldsymbol{e}_\mu^{T}  \nabla F\left(\boldsymbol{g}_{\mu} \right) \right)\\
&+\frac  {1} {2L} \left(   \| \nabla F\left(\boldsymbol{g}_{\mu} \right)\|^2-2\boldsymbol e_\mu^{T} \nabla F\left(\boldsymbol{g}_{\mu}\right)+\| \boldsymbol e_\mu\|^2 \right),\\
 = & \mathbb E\left(F \left(\boldsymbol{g}_{\mu} \right)\right)-\frac{1}{2L}\|  \nabla F\left(\boldsymbol{g}_{\mu}\right) \|^2 +\frac{1}{2L}\mathbb E\left( \|\boldsymbol{e}_\mu\|^2 \right).
\end{split}
\end{equation}
Let $\mathcal{N}_1=\left\{a_{i,\mu}=1 | i \in \mathcal{U} \right\}$ be the set of selected users that transmit their local
FL models to the BS, $\mathcal{N}_2=\left\{a_{i,\mu}=0, E_{i,\mu}\le \gamma | i \in \mathcal{U} \right\}$ be the set of users that the BS can accurately estimate their local FL models, and $\mathcal{N}_3=\left\{i \in \mathcal{U}| i \notin \mathcal{N}_1,i \notin \mathcal{N}_2 \right\}$. Meanwhile, let $A={\sum\limits_{i=1}^U K_i\left(1-a_{i,\mu}\right)\mathbbm{1}_{\left\{ E_{i,\mu}\le \gamma \right\}}+\sum\limits_{i=1}^U K_ia_{i,\mu}}$. Given (\ref{eq:errormu}), we have
%In (\ref{itproofas3_2}), $\mathbb E \left(\| \boldsymbol e_\mu \|^2\right)$ depends on the user selection scheme and the predictions of users' local FL model. Since the predicted local FL models will be used to update the global FL model once the prediction accuracy reaches the system requirement, there exists a number of iterations in which the predicted local FL models will not be used to update the global FL model. Hence, it is not easy to directly derive the convergence value and rate of the proposed FL. Next, we first derive the convergence value and rate of algorithm a). Then, the convergence of algorithm b) is analyzed.   

\begin{equation}\label{eq:eb}\small
\begin{split}
&\mathbb E\left(\| \boldsymbol{e}_\mu \|^2\right)=\\&\mathbb E\left(\left\|\nabla F\left(\boldsymbol{g}_{\mu} \right) -    \frac{ \sum\limits_{i=1}^U K_i\left(1-a_{i,\mu}\right)\left( \nabla F_i\left(\boldsymbol{g}_{\mu} \right)+ \nabla \hat F_i\left( \boldsymbol{g}_{\mu}\right) \right)\mathbbm{1}_{\left\{ E_{i,\mu}\le \gamma \right\}}}{\sum\limits_{i=1}^U K_i\left(1-a_{i,\mu}\right)\mathbbm{1}_{\left\{ E_{i,\mu}\le \gamma \right\}}+\sum\limits_{i=1}^U K_ia_{i,\mu}} \right. \right. \\& \left. \left. - \frac{\sum\limits_{i=1}^U K_ia_{i,\mu}\nabla F_i\left(\boldsymbol{g}_{\mu}\right) }{\sum\limits_{i=1}^U K_i\left(1-a_{i,\mu}\right)\mathbbm{1}_{\left\{ E_{i,\mu}\le \gamma \right\}}+\sum\limits_{i=1}^U K_ia_{i,\mu}}   \right\|^2\right), \\ 
&=\mathbb E \left(\left\|-\frac{ \left(K- A  \right)\sum\limits_{i \in \mathcal{N}_1}  {K_i\nabla F_i\left( {\boldsymbol{g}_\mu} \right)}}{{KA }} -\frac{\left(K-A \right) \sum\limits_{i \in \mathcal{N}_2} K_i {\nabla F_i\left( {\boldsymbol{g}_\mu} \right)}}{KA} \right. \right. \\& \left. \left. -\frac{\sum\limits_{i \in \mathcal{N}_2} K_i {\nabla \hat F_i\left( {\boldsymbol{g}_\mu} \right)}}{A}+\frac{\sum\limits_{i \in \mathcal{N}_3} K_i {\nabla F_i\left( {\boldsymbol{g}_\mu} \right)}}{K} \right\|^2\right),\\
&\leq\mathbb E \left(\frac{ \left(K- A  \right)\sum\limits_{i \in \mathcal{N}_1}  {K_i \left\| \nabla F_i\left( {\boldsymbol{g}_\mu} \right)\right\| }}{{KA }}+\frac{\left(K-A \right) \sum\limits_{i \in \mathcal{N}_2} K_i {\left\| \nabla F_i\left( {\boldsymbol{g}_\mu} \right)\right\| }}{KA} \right. \\& \left. +\frac{\sum\limits_{i \in \mathcal{N}_2} K_i {\left\| \nabla \hat F_i\left( {\boldsymbol{g}_\mu} \right) \right\|}}{A}+\frac{\sum\limits_{i \in \mathcal{N}_3} K_i {\left\| \nabla F_i\left( {\boldsymbol{g}_\mu} \right) \right\|}}{K}\right)^2.
\end{split}
 \end{equation}
Since $\| \nabla F_i\left(\boldsymbol{g}_\mu\right) \|^2 \le\zeta_{\mu}^1+\zeta_{\mu}^2 \| \nabla F \left(\boldsymbol{g}_{\mu} \right)\|^2$, $\| \nabla \hat F_i\left(\boldsymbol{g}_{\mu} \right)\|^2 \leq \zeta_{\mu}^1+\zeta_{\mu}^2 \| \nabla F \left(\boldsymbol{g}_{\mu} \right)\|^2$, (\ref{eq:eb}) can be simplified by
 \begin{equation}\label{eq:error}\small
\begin{split}
&\mathbb E\left(\| \boldsymbol{e}_\mu \|^2\right)\leq\mathbb E \left(\frac{ \left(K- A  \right)\sqrt{\zeta_{\mu}^1+\zeta_{\mu}^2 \| \nabla F \left(\boldsymbol{g}_{\mu} \right)\|^2}}{{K}} \right. \\& \left.
+\frac{\sum\limits_{i \in \mathcal{N}_2} K_i {\left\| \nabla \hat F_i\left( {\boldsymbol{g}_\mu} \right) \right\|}}{A}+\frac{\left(K-A\right)\sqrt{\zeta_{\mu}^1+\zeta_{\mu}^2 \| \nabla F \left(\boldsymbol{g}_{\mu} \right)\|^2}}{K}\right)^2\\
\leq&\mathbb E \left(\frac{ 2\left(K- A  \right)\sqrt{\zeta_{\mu}^1+\zeta_{\mu}^2 \| \nabla F \left(\boldsymbol{g}_{\mu} \right)\|^2}}{{K}}\right. \\& \left.
+{\sqrt{\zeta_{\mu}^1+\zeta_{\mu}^2 \| \nabla F \left(\boldsymbol{g}_{\mu} \right)\|^2}} \right)^2,\\
=&\mathbb E \left(\frac{ 4\left(K- A  \right)^2{\left( \zeta_{\mu}^1+\zeta_{\mu}^2 \| \nabla F \left(\boldsymbol{g}_{\mu} \right)\|^2 \right)}}{{K^2}}+\zeta_{\mu}^1+\zeta_{\mu}^2 \| \nabla F \left(\boldsymbol{g}_{\mu} \right)\|^2\right. \\& \left.
+\frac{4\left(K-A\right) {\left(\zeta_{\mu}^1+\zeta_{\mu}^2 \| \nabla F \left(\boldsymbol{g}_{\mu} \right)\|^2 \right)}}{K} \right),\\
\le&\mathbb E \left( \frac{\left( 9K-8A\right) \left(\zeta_{\mu}^1+\zeta_{\mu}^2 \| \nabla F \left(\boldsymbol{g}_{\mu} \right)\|^2 \right) }{K}\right)\\
=&\frac{\left(9K- 8\mathbb E\left(A\right)\right)}{K} \left(\zeta_{\mu}^1+\zeta_{\mu}^2 \| \nabla F \left(\boldsymbol{g}_{\mu} \right)\|^2 \right).
%\leq&\mathbb E\frac{ 3\left(K- \frac{2A}{3}\right)\left( {\zeta_{\mu}^1+\zeta_{\mu}^2 \| \nabla F \left(\boldsymbol{g}_{\mu} \right)\|\right)}}{{K}},
%\le&\mathbb E\left(\left\|  \nabla F (\boldsymbol{g}_{\mu} )-\frac{ \sum\limits_{i=1}^U K_i\left(1-a_{i,\mu}\right)\left( \nabla F_i(\boldsymbol{g}_{\mu} )+ \nabla \hat F_i\left( \boldsymbol{g}_{\mu}\right) \right)\mathbbm{1}_{\left\{ E_{j,\mu}\le \gamma \right\}} +\sum\limits_{i=1}^U K_ia_{i,\mu}\nabla F_i(\boldsymbol{g}_{\mu} ) }{K}  \right\|^2 \right),\\
%=&\mathbb E\left(\left\|\frac{\sum\limits_{i=1}^U K_i\left(1-a_{i,\mu}\right)\nabla F_i(\boldsymbol{g}_{\mu} )- \sum\limits_{i=1}^U K_i\left(1-a_{i,\mu}\right)\left( \nabla F_i(\boldsymbol{g}_{\mu} )+ \nabla \hat F_i \left( \boldsymbol{g}_{\mu}\right) \right)\mathbbm{1}_{\left\{ E_{j,\mu}\le \gamma \right\}}}{K}\right\|^2\right),\\
%\leq&\frac{1}{K} \sum\limits_{i=1}^U K_i \mathbb E\left(1-a_{i,\mu}\right)\left\|\nabla F_i(\boldsymbol{g}_{\mu} )\right\|^2\\&-\frac{1}{K}\sum\limits_{i=1}^U K_i\mathbb E\left(1-a_{i,\mu}\right)\left(\left\| \nabla F_i(\boldsymbol{g}_{\mu} )\right\|^2+ \left\| \nabla \hat F_i \left( \boldsymbol{g}_{\mu}\right) \right\|^2 \right)\mathbbm{1}_{\left\{ E_{j,\mu}\le \gamma \right\}}.
\end{split}
 \end{equation}
Since $\mathbb E\left(1-a_{i,\mu}\right)=1-p_{i,\mu}$ and $\mathbb E\left(a_{i,\mu}\right)=p_{i,\mu}$, we have $\mathbb E\left(A\right)={\sum\limits_{i=1}^U K_i\left(1-p_{i,\mu}\right)\mathbbm{1}_{\left\{ E_{i,\mu}\le \gamma \right\}}+\sum\limits_{i=1}^U K_ip_{i,\mu}}$.
 
Substitute (\ref{eq:error}) into (\ref{itproofas3_2}), we have
\begin{equation}\label{eq:EF}
\begin{split}
\mathbb E \left[ F\left(\boldsymbol{g}_{\mu+1}\right)\right] \leq& \mathbb E\left(F \left(\boldsymbol{g}_{t} \right)\right)+\frac{\zeta_{\mu}^1\left(9K- 8\mathbb E\left(A\right)\right)}{2LK}\\
&-\frac{1}{2L} \left(1- \frac{\zeta_{\mu}^2\left(9K- 8\mathbb E\left(A\right)\right) }{K} \right)   \| \nabla F\left(\boldsymbol{g}_{t} \right)\|^2.
\end{split}
\end{equation}

Subtract $\mathbb E\left[ F\left(\boldsymbol{g}^{*}\right)\right]$ in both sides of (\ref{eq:EF}), we have
\begin{equation}\label{eq:EF2}
\begin{split}
&\mathbb E\left[ F\left(\boldsymbol{g}_{\mu+1}\right)-F\left(\boldsymbol{g}^{*}\right)\right]\\& \leq\mathbb E\left(F \left(\boldsymbol{g}_{\mu} \right)-F\left(\boldsymbol{g}^{*}\right)\right)
+\frac{\zeta_{\mu}^1\left(9K- 8\mathbb E\left(A\right)\right)}{2LK}\\
&-\frac{1}{2L} \left(1- \frac{\zeta_{\mu}^2\left(9K- 8\mathbb E\left(A\right)\right) }{K} \right)   \| \nabla F\left(\boldsymbol{g}_{t}\right)\|^2.
\end{split}
\end{equation}
Given (\ref{eq:Fg}), (\ref{eq:EF2}) can be rewritten as
\begin{equation}\label{eq:EF3}
\begin{split}
\mathbb E\left[ F\left(\boldsymbol{g}_{\mu+1}\right)-F\left(\boldsymbol{g}^{*}\right)\right]  \leq &\frac{\zeta_{\mu}^1\left(9K- 8\mathbb E\left(A\right)\right)}{2LK}\\
&\!\!\!\!\!\!\!\!\!\!\!\!\!\!\!\!\!\!\!\!\!\!\!\!\!\!\!\!\!\!\!\!\!\!\!\!\!\!\!\!\!\!\!\!\!\!\!\!\!\!\!\!\!\!+\left(1-\frac{\vartheta}{L}+\frac{\vartheta\zeta_{\mu}^2\left(9K- 8\mathbb E\left(A\right)\right)}{LK}\right)\mathbb E\left(F\left(\boldsymbol{g}_{\mu}\right)-F\left(\boldsymbol{g}^{*}\right)\right).
\end{split}
\end{equation}
This completes the proof.

\subsection{Proof of Theorem \ref{th:2}}\label{Ap:a}
{\color{black}To prove Theorem \ref{th:2}, we only need to derive $\mathbb E\left(\| \boldsymbol{e}_\mu \|^2\right)$ in (\ref{itproofas3_2}). For the proposed FL algorithm with SGD, the gradient deviation $\boldsymbol{e}_\mu^\textrm{SGD}$ caused by the users that do not transmit their local FL models and the SGD training method is given by

\begin{equation}\label{eq:errormuSGD}\small
\begin{split}
\boldsymbol{e}_\mu^\textrm{SGD}&=\nabla F\left(\boldsymbol{g}_{\mu}\right)\\
&-\frac{{\!\!\sum\limits_{i=1}^U\! \left(1-a_{i,\mu}\right)\!K_i\!\left( \nabla F_i\left(\boldsymbol{g}_{\mu}\right)\!+\! \nabla \hat F_i\left( \boldsymbol{g}_{\mu}\right)\! \right)\!\!\mathbbm{1}_{\left\{ E_{i,\mu}\le \gamma \right\}} \! }}{\sum\limits_{i=1}^U K_i\left(1-a_{i,\mu}\right)\mathbbm{1}_{\left\{ E_{i,\mu}\le \gamma \right\}}+\sum\limits_{i=1,i \ne i^*}^U\sum\limits_{k \in \mathcal{K}_{i,\mu}}  a_{i,\mu}+K_{i^*}}\\
&-\frac{{\!\!\!\!\sum\limits_{i=1,i\ne i^*}^U\sum\limits_{k \in \mathcal{K}_{i,\mu}}a_{i,\mu}\nabla {f_{ik}\left( { \boldsymbol{ g}_\mu} \right)}  }-K_i^*\nabla F_{i^*}\left(\boldsymbol{g}_{\mu} \right)}{\sum\limits_{i=1}^U K_i\left(1-a_{i,\mu}\right)\mathbbm{1}_{\left\{ E_{i,\mu}\le \gamma \right\}}+\sum\limits_{i=1,i \ne i^*}^U\sum\limits_{k \in \mathcal{K}_{i,\mu}}  a_{i,\mu}+K_{i^*}},
\end{split}
\end{equation}
%where $\nabla {f\left( { \boldsymbol{ g}\left(\boldsymbol{a}_\mu\right)} \right)}$ is short for $\nabla {f\left( { \boldsymbol{ g}\left(\boldsymbol{a}_\mu\right) ,{\boldsymbol{x}_{ik}},{\boldsymbol{y}_{ik}}} \right)}$. 
In (\ref{eq:errormuSGD}), $\sum\limits_{i=1,i\ne i^*}^U a_{i,\mu}\nabla {f\left( { \boldsymbol{ g}\left(\boldsymbol{a}_\mu\right)} \right)} $ is the sum of the local FL models that are transmitted from the users, except user $i^*$, that have RBs. Since the BS needs to use the local FL model parameters of user $i^*$ for the predictions of other users' local FL models, user $i^*$'s local FL model must be trained by all of its collected data. Hence, $K_i^*\nabla F_{i^*}\left(\boldsymbol{g}_{\mu}\right)$ is the local model that is transmitted by user $i^*$. Given (\ref{eq:errormuSGD}) and $\| \nabla F_i\left(\boldsymbol{g}_\mu\right) \|^2, \| \nabla \hat F_i\left(\boldsymbol{g}_\mu\right) \|^2, \| \nabla f_{ik}\left(\boldsymbol{g}_\mu\right) \|^2 \le\zeta_{\mu}^1+\zeta_{\mu}^2 \| \nabla F \left(\boldsymbol{g}_{\mu} \right)\|^2$, $\left\|\boldsymbol{e}_\mu^\textrm{SGD} \right\|$ can be expressed by
\begin{equation}\label{eq:sgde}\small
\begin{split}
\left\|\boldsymbol{e}_\mu^\textrm{SGD} \right\|\leq&\frac{\left(K-A^\textrm{S}\right)K_i^*\left\|\nabla F_{i^*}\left(\boldsymbol{g}_{\mu} \right)\right\|}{KA^\textrm{S}}+\frac{\left(K-A^\textrm{S}\right)\sum\limits_{i \in \mathcal{N}_2}K_i \left\| \nabla F_{i}\left(\boldsymbol{g}_{\mu} \right) \right\|}{KA^\textrm{S}}\\
&+\frac{\sum\limits_{i \in \mathcal{N}_2}K_i\left\| \nabla \hat F_{i}\left(\boldsymbol{g}_{\mu}\right)\right\|}{A^\textrm{S}}+\frac{\left(K-A^\textrm{S}\right)\sum\limits_{i \in \mathcal{N}_1}\sum\limits_{k \in \mathcal{K}_{i,\mu}} \left\| \nabla f_{ik}\left(\boldsymbol{g}_{\mu} \right) \right\|}{KA^\textrm{S}}\\
&+
\frac{\sum\limits_{i \in \mathcal{N}_3}K_i \left\| \nabla F_{i}\left(\boldsymbol{g}_{\mu} \right) \right\|}{K}+\frac{\sum\limits_{i \in \mathcal{N}_1}\sum\limits_{k \in \mathcal{K}_{i}/\mathcal{K}_{i,\mu}} \left\| \nabla f_{ik}\left(\boldsymbol{g}_{\mu} \right) \right\|}{K},\\
\leq& \frac{2\left(K-A^\textrm{S}\right)\sqrt{\zeta_{\mu}^1+\zeta_{\mu}^2 \| \nabla F \left(\boldsymbol{g}_{\mu} \right)\|^2}}{K}+\sqrt{\zeta_{\mu}^1+\zeta_{\mu}^2 \| \nabla F \left(\boldsymbol{g}_{\mu} \right)\|^2},\\
\end{split}
\end{equation}
where the second inequality stems from the fact that $\frac{\left(K-A^\textrm{S}\right)K_i^*\left\|\nabla F_{i^*}\left(\boldsymbol{g}_{\mu} \right)\right\|}{KA^\textrm{S}}+\frac{\left(K-A^\textrm{S}\right)\sum\limits_{i \in \mathcal{N}_2}K_i \left\| \nabla F_{i}\left(\boldsymbol{g}_{\mu} \right) \right\|}{KA^\textrm{S}}+\frac{\left(K-A^\textrm{S}\right)\sum\limits_{i \in \mathcal{N}_1}\sum\limits_{k \in \mathcal{K}_{i,\mu}} \left\| \nabla f_{ik}\left(\boldsymbol{g}_{\mu} \right) \right\|}{KA^\textrm{S}}= \frac{\sum\limits_{i \in \mathcal{N}_3}K_i \left\| \nabla F_{i}\left(\boldsymbol{g}_{\mu} \right) \right\|}{K}+\frac{\sum\limits_{i \in \mathcal{N}_1}\sum\limits_{k \in \mathcal{K}_{i}/\mathcal{K}_{i,\mu}} \left\| \nabla f_{ik}\left(\boldsymbol{g}_{\mu} \right) \right\|}{K}= \frac{\left(K-A^\textrm{S}\right)\sqrt{\zeta_{\mu}^1+\zeta_{\mu}^2 \| \nabla F \left(\boldsymbol{g}_{\mu} \right)\|^2}}{K}$ and $\frac{\sum\limits_{i \in \mathcal{N}_2}K_i\left\| \nabla \hat F_{i}\left(\boldsymbol{g}_{\mu} \right)\right\|}{A^\textrm{S}} \leq \sqrt{\zeta_{\mu}^1+\zeta_{\mu}^2 \| \nabla F \left(\boldsymbol{g}_{\mu} \right)\|^2}$.

Based on (\ref{eq:error}) and (\ref{eq:sgde}), we have $ 
 \mathbb E\left(\| \boldsymbol{e}_\mu^\textrm{SGD} \|^2\right)=\frac{\left(9K- 8\mathbb E\left(A^\textrm{S}\right)\right)}{K} \left(\zeta_{\mu}^1+\zeta_{\mu}^2 \| \nabla F \left(\boldsymbol{g}_{\mu} \right)\|^2 \right)$ with $\mathbb E\left(A^\textrm{S}\right)=\sum\limits_{i=1}^U K_i\left(1-p_{i,\mu}\right)\mathbbm{1}_{\left\{ E_{i,\mu}\le \gamma \right\}}+\sum\limits_{i=1,i \ne i^*}^U\sum\limits_{k \in \mathcal{K}_{i,\mu}}  p_{i,\mu}+K_{i^*}$.

 Based on (\ref{eq:EF}) to (\ref{eq:EF3}), we have
\begin{equation}\label{eq:EF3}
\begin{split}
\mathbb E \left[ F\left(\boldsymbol{g}_{\mu+1}\right)-F\left(\boldsymbol{g}^{*}\right)\right]  \leq &\frac{\zeta_{\mu}^1\left(9K- 8\mathbb E\left(A^\textrm{S}\right)\right)}{2LK}\\
&\!\!\!\!\!\!\!\!\!\!\!\!\!\!\!\!\!\!\!\!\!\!\!\!\!\!\!\!\!\!\!\!\!\!\!\!\!\!\!\!\!\!\!\!\!\!\!\!+\left(1-\frac{\vartheta}{L}+\frac{\vartheta\zeta_{\mu}^2\left(9K- 8\mathbb E\left(A^\textrm{S}\right)\right)}{LK}\right)\mathbb E\left(F\left(\boldsymbol{g}_{\mu} \right)-F\left(\boldsymbol{g}^{*}\right)\right).
\end{split}
\end{equation}
This completes the proof.}

\bibliographystyle{IEEEbib}
\bibliography{references1}
\end{document}